%% file: main.tex
\documentclass[sigconf]{acmart}

\copyrightyear{2025}
\acmYear{2025}
\setcopyright{acmlicensed}
\acmConference[WWW '25] {Proceedings of the ACM Web Conference 2025}{April 28--May 2, 2025}{Sydney, NSW, Australia.}
\acmBooktitle{Proceedings of the ACM Web Conference 2025 (WWW '25), April 28--May 2, 2025, Sydney, NSW, Australia}
\acmISBN{979-8-4007-1274-6/25/04}
\acmDOI{10.1145/3696410.3714916}

\usepackage{bbm}
\usepackage{subfigure}
\usepackage{multirow}
\usepackage{bm}
\usepackage{enumitem}
\usepackage{graphicx}
\usepackage{algorithmic}
\usepackage[ruled,vlined,linesnumbered]{algorithm2e}
\usepackage{enumitem}
\newtheorem{definition}{Definition}
\newtheorem{proposition}{Proposition}
\newtheorem{remark}{Remark}
\usepackage[normalem]{ulem}
\useunder{\uline}{\ul}{}
\usepackage{makecell}
\usepackage{etoc}
\etocdepthtag.toc{mtchapter}
\etocsettagdepth{mtchapter}{subsection}
\etocsettagdepth{mtappendix}{none}

\begin{document}

\title{Rethinking and Accelerating Graph Condensation: A Training-Free Approach with Class Partition}

\author{Xinyi Gao}
\affiliation{
\institution{The University of Queensland}
\city{Brisbane}
\country{Australia}
}
\email{xinyi.gao@uq.edu.au}

\author{Guanhua Ye}
\authornote{Corresponding author.}
\affiliation{
\institution{Beijing University of Posts and Telecommunications}
\city{Beijing}
\country{China}
}
\email{g.ye@bupt.edu.cn}

\author{Tong Chen}
\affiliation{
\institution{The University of Queensland}
\city{Brisbane}
\country{Australia}
}
\email{tong.chen@uq.edu.au}

\author{Wentao Zhang}
\affiliation{
\institution{Peking University}  
\city{Beijing}
\country{China}
}
\email{wentao.zhang@pku.edu.cn}

\author{Junliang Yu}
\affiliation{
\institution{The University of Queensland}
\city{Brisbane}
\country{Australia}
}
\email{jl.yu@uq.edu.au}

\author{Hongzhi Yin}
\authornotemark[1]
\affiliation{
\institution{The University of Queensland}  
\city{Brisbane}
\country{Australia}
}
\email{db.hongzhi@gmail.com}

\renewcommand{\shortauthors}{Xinyi Gao et al.}

\begin{abstract}
The increasing prevalence of large-scale graphs poses a significant challenge for graph neural network training, attributed to their substantial computational requirements. In response, graph condensation (GC) emerges as a promising data-centric solution aiming to substitute the large graph with a small yet informative condensed graph to facilitate data-efficient GNN training. However, existing GC methods suffer from intricate optimization processes, necessitating excessive computing resources and training time. In this paper, we revisit existing GC optimization strategies and identify two pervasive issues therein: (1) various GC optimization strategies converge to coarse-grained class-level node feature matching between the original and condensed graphs; (2) existing GC methods rely on a Siamese graph network architecture that requires time-consuming bi-level optimization with iterative gradient computations. To overcome these issues, we propose a training-free GC framework termed Class-partitioned Graph Condensation (CGC), which refines the node distribution matching from the class-to-class paradigm into a novel class-to-node paradigm, transforming the GC optimization into a class partition problem which can be efficiently solved by any clustering methods. Moreover, CGC incorporates a pre-defined graph structure to enable a closed-form solution for condensed node features, eliminating the need for back-and-forth gradient descent in existing GC approaches. Extensive experiments demonstrate that CGC achieves an exceedingly efficient condensation process with advanced accuracy. Compared with the state-of-the-art GC methods, CGC condenses the Ogbn-products graph within 30 seconds, achieving a speedup ranging from $10^2 \times$ to $10^4 \times$ and increasing accuracy by up to 4.2\%. The code is available at: \href{https://github.com/XYGaoG/CGC}{https://github.com/XYGaoG/CGC}. 
\end{abstract}


  
\maketitle

\input{1int}
\input{2ret}

\input{3met}

\input{4exp}
\input{6con}

\bibliographystyle{ACM-Reference-Format}
\balance
\bibliography{ref}

\input{7app}

\input{7app2}

\input{7app3}

\end{document}

%% file: 1int.tex
\section{Introduction}
\label{sec_intro}

Graph neural networks (GNNs) \cite{huang2020temporal, zheng2022rethinking,huang2024cost,gao2023semantic} have garnered significant attention for their exceptional representation capabilities for complex graph data and have been utilized in a wide range of real-world applications, including chemical molecules~\cite{sun2019infograph}, social networks~\cite{huang2022influence}, and recommender systems~\cite{yu2023self}. 
However, the increasing prevalence of large-scale graphs within these real-world applications poses formidable challenges in training GNN models. 
Most GNNs follow the message-passing paradigm~\cite{gilmer2017neural}, which is formulated as convolutions over the entire graph and aggregating information from multi-hop neighboring nodes. 
This process leads to exponential growth in neighbor size~\cite{hamilton2017inductive,gao2024accelerating} when applied to large-scale graphs, necessitating considerable training computations.
In response to the urgent demand for processing large-scale graphs, a few studies borrow the idea of dataset distillation~\cite{wang2018dataset,lei_comprehensive_2024} from computer vision (CV) and introduce graph condensation (GC)~\cite{jin2022graph,gao2024graph} to generate a compact yet informative graph to accelerate the GNN training.
By capturing essential characteristics of the original large graph, GNNs trained on these small condensed graphs can achieve comparable performance to those trained on the original graphs. This efficacy enables GC to be applied to a variety of applications rapidly, e.g., graph continual learning~\cite{liu_cat_2023}, inference acceleration~\cite{gao_graph_2023}, and hyper-parameter search~\cite{ding_faster_2022}.

Despite the effectiveness of expediting GNN training, existing GC practices still suffer from complex optimization and intensive computation during the condensation process.
As depicted in Figure \ref{fig_main} (a), to bridge the original and condensed graph, existing GC methods employ the Siamese network architecture to encode both graphs through a relay model, and the condensed graph is optimized to simulate the class distributions of the original graph.
This framework necessitates a bi-level optimization procedure, with the inner loop refining the relay model on the condensed graph and the outer loop subsequently optimizing the condensed graph, ensuring optimal model performance on both graphs.
However, two main issues persist in this framework:
(1) Existing optimization strategies in GC~\cite{jin2022graph,zheng_structure_free_2023,wang2023fast,liu2022graph,liu_graph_2023,xiao2024simple} manifest a single, unified optimization objective for all condensed nodes within the same class, resulting in the coarse-grained optimization target for condensed nodes.
(2) Bi-level optimization involves iterative and intensive gradient computations for both the relay model and condensed graph.
Recent efforts aim to expedite this process by simplifying outer or inner loop optimizations.
Techniques like distribution matching~\cite{liu2022graph} and the structure-free approach~\cite{zheng_structure_free_2023} respectively eliminate model gradient calculations and adjacency matrix optimizations in the outer loop.
Meanwhile, approaches such as one-step matching~\cite{jin2022condensing}, kernel ridge regression (KRR)~\cite{wang2023fast} and pre-trained model~\cite{xiao2024simple} simplify the relay model updates in the inner loop. 
Despite these advancements, the optimization of the condensed graph still involves back-and-forth gradient calculations and updates, resulting in a time-consuming condensation procedure.

\begin{figure}[t]
\centering
\includegraphics[width=0.95\linewidth]{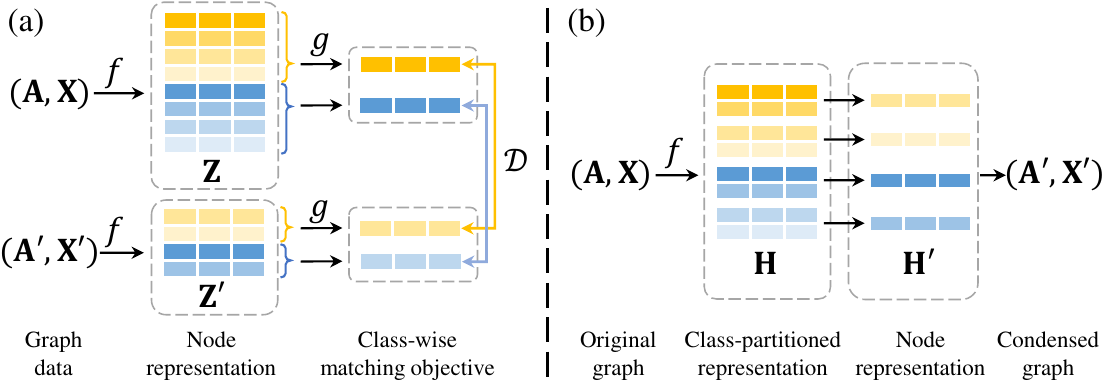}
\caption{(a) The class-to-class matching paradigm in existing GC methods. (b) Our proposed class-to-node matching paradigm. $f$ denotes the relay model, $\mathcal{D}$ represents the distance function, and $g$ measures the matching objective (refer to Table \ref{tab_GCcom} for details).}
\label{fig_main}
\end{figure}

To address these issues, we investigate the foundational objective of existing optimization strategies and design an exceedingly efficient GC approach with a training-free framework termed \underline{C}lass-partitioned \underline{G}raph \underline{C}ondensation (CGC). As illustrated in Figure \ref{fig_main}(b), CGC refines the distribution matching from the class-to-class paradigm to a delicate class-to-node distribution matching paradigm. Notably, this refinement further simplifies the distribution matching objective as a class partition problem, which can be efficiently optimized using any clustering method (e.g., K-Means).
Moreover, CGC utilizes the pre-defined graph structure with the Dirichlet energy constraint~\cite{kalofolias2016learn} to derive a closed-form solution for the condensed node features.
Equipped with these non-parametric modules, CGC eliminates gradient-based optimization in existing GC methods, enabling the condensation process to be executed on CPUs only. This enhances both the efficiency and effectiveness of the GC process, significantly broadening the utility of GC in real-world applications.

The main contributions of this paper are three-fold:

\begin{itemize}[leftmargin=*]
\item \textbf{New observations and insights.}
We theoretically demonstrate that existing GC optimization strategies converge to the class-level distribution matching paradigm, and subsequently simplify this process into a class partition problem, eliminating the sophisticated bi-level optimization and enabling efficient resolution through any clustering method.
\item \textbf{New GC framework.}
We present CGC, the first training-free GC framework characterized by a fine-grained class-to-node matching paradigm and closed-form feature generation, facilitating both precise and efficient GC procedures. Furthermore, CGC demonstrates considerable versatility, and any component within the framework can be replaced with a variety of alternative methods, such as distinct propagation and partition techniques.
\item \textbf{State-of-the-art performance.}
Extensive experiments demonstrate that CGC achieves advanced accuracy with an extremely fast condensation procedure. For instance, it condenses the Ogbn-products dataset within 30 seconds, which is 1,038$\times$ faster than GCond~\cite{jin2022graph} and 148$\times$ faster than the most efficient GC baseline SimGC~\cite{xiao2024simple}.
\end{itemize}

%% file: 2ret.tex
\section{Rethinking Existing Graph Condensation Methods}
\label{sec_rethink}

\begin{table*}[t]
\setlength{\abovecaptionskip}{0.1cm}
\centering
\caption{The comparison of different GC optimization strategies.}
\label{tab_GCcom}
\resizebox{0.8\linewidth}{!}{
\begin{tabular}{llll}
\toprule
Optimization strategy         & Representative methods           & $f$               & $g$                        \\ \midrule
Parameter matching      & GCond~\cite{jin2022graph}, SFGC~\cite{zheng_structure_free_2023},  GEOM~\cite{zhang2024navigating}, GCSR~\cite{liu2024graph}                  & GNN                 & Classification model parameter \\
Performance matching    & GC-SNTK~\cite{wang2023fast}, KiDD~\cite{xu2023kernel}                             & Graph kernel        & Regression model parameter        \\
Distribution matching   & GCDM~\cite{liu2022graph}, GCEM~\cite{liu_graph_2023}, SimGC~\cite{xiao2024simple}     & GNN                 & Class prototype            \\
 \bottomrule
\end{tabular}}
\end{table*}

In this section, we first formally formulate the graph condensation and then revisit the existing GC optimization strategies along with their mutual connections.
Subsequently, we formulate these strategies within a unified framework and demonstrate their adherence to a common class-level distribution matching paradigm.
Finally, we simplify this paradigm into a class partition problem, enabling efficient optimization through clustering methods.

\subsection{Problem Formulation}
We consider a large-scale original graph $\mathcal{T}=\{{\bf A}, {\bf X}\}$ with $N$ nodes, where ${\bf A}\in \mathbb{R}^{N\times N}$ is the adjacency matrix and ${\bf X}\in{\mathbb{R}^{N\times d}}$ denotes the $d$-dimensional node feature matrix.
Each node belongs to one of $c$ classes $\{C_1, \cdots, C_c\}$, translating into numeric labels ${\bf y}\in{\{1, ..., c\}^{N}}$ and one-hot labels ${\bf Y}\in{\mathbb{R}^{N\times c}}$. 
Graph condensation~\cite{jin2022graph} aims to generate a small condensed graph $\mathcal{S}=\{{\bf A'}, {\bf X'}\}$, such that GNNs trained on $\mathcal{S}$ yield performance comparable to those trained on $\mathcal{T}$.
Specifically, ${\bf A'}\in\mathbb{R}^{N'\times N'}$, ${\bf X'}\in\mathbb{R}^{N'\times d}$ and $c \leq N'\ll{N}$. 
Similarly, each node in $\mathcal{S}$ belongs to one of $c$ classes $\{C'_1, \cdots, C'_c\}$, and labels are denoted as ${\bf y'} \in {\{1, ..., c\}^{N'}}$ or ${\bf Y'} \in \mathbb{R}^{N' \times c}$. We follow GCond to pre-define the condensed node labels, which preserve the same class proportion as the original node labels.
To facilitate the expression, we assume all nodes in the original and condensed graphs are organized in ascending order according to labels. 

Notice that the generation of ${\bf A'}$ is optional in existing GC methods~\cite{jin2022graph,zhao2023dataset,zheng_structure_free_2023}. 
If ${\bf A'}$ is opted out, the identity matrix ${\bf I}$ is used instead in GNN training, and this approach is termed the graphless GC variant (a.k.a. structure-free GC~\cite{zheng_structure_free_2023}).

\subsection{Class-level Matching Paradigm in GC}

To achieve the GC objective, existing methods use a Siamese network architecture with a relay model $f$ to encode both graphs as shown in Figure \ref{fig_main} (a) and employ three kinds of optimization strategies~\cite{yu2023dataset,xu2024survey,hashemi2024comprehensive}: parameter matching, performance matching, and distribution matching.

\textbf{Parameter matching} posits that the parameters of the GNN classifier should possess high consistency whenever it is trained on $\mathcal{S}$ or $\mathcal{T}$.
To this end, GCond~\cite{jin2022graph} tries to match model parameters at each training step and simplifies the objective to facilitate that gradients generated by $\mathcal{S}$ match those from the same class in $\mathcal{T}$:
\begin{equation}
\label{eq_GMeq}
\begin{aligned}
\mathcal{L}_{GM}  &= {\mathop{\mathbb{E}}_{{\mathbf{\mathbf{\Theta}} } \sim \Phi}} \left[\sum_{i=1}^{c} \mathcal{D}( \nabla_{{\mathbf{\Theta}}} \mathcal{L}_{i}^{\mathcal{S}},\nabla_{{\mathbf{\Theta}}} \mathcal{L}_{i}^{\mathcal{T}})\right],\\
\end{aligned}
\end{equation}
where $\Phi$ is the distribution of the relay model (i.e., the GNN classifier) parameter ${\mathbf{\Theta}}$, and we omit relay model update in the inner loop for simplicity. $\mathcal{D}$ indicates the distance function.
$\mathcal{L}_{i}^{\mathcal{S}}$ and $\mathcal{L}_{i}^{\mathcal{T}}$ are classification losses (e.g., cross-entropy loss) for class $i$ w.r.t  $\mathcal{S}$ and $\mathcal{T}$, respectively. 
However, gradient matching~\cite{zhao2020dataset} may accumulate errors when the relay model is iteratively updated on $\mathcal{S}$ over multiple steps. 
To mitigate this problem, SFGC~\cite{zheng_structure_free_2023} introduces trajectory matching~\cite{cazenavette2022dataset} to align the long-term training trajectories of classification models.
Nonetheless, to avoid overfitting one initialization of the relay model, it requires training hundreds of GNNs on $\mathcal{T}$ to obtain the trajectories, resulting in heavy condensation computations.

\textbf{Performance matching}
aligns the performance of models trained on $\mathcal{S}$ and $\mathcal{T}$ by ensuring that the model trained on $\mathcal{S}$ achieves minimal loss on $\mathcal{T}$~\cite{yu2023dataset}.
To obtain the optimal model on $\mathcal{S}$, KiDD~\cite{xu2023kernel} and GC-SNTK~\cite{wang2023fast} substitute the classification task with the regression and incorporate KRR~\cite{nguyen2021dataset} in the GC procedure for a closed-form solution of the relay model. The objective is formulated as:
\begin{equation}
\label{eq_PMeq}
\mathcal{L}_{PM} =  \left \| {\bf Y} - {\bf Z} {\bf Z}'^{\top}\left({\bf Z}'{\bf Z}'^{\top}+\lambda{\bf I} \right)^{-1} {\bf Y}'  \right \|^{2},
\end{equation}
where ${\bf Z}'\in \mathbb{R}^{N'\times d}$ and ${\bf Z}\in \mathbb{R}^{N\times d}$ are node embeddings for $\mathcal{S}$ and $\mathcal{T}$, respectively. $\lambda$ is a small constant weight of the regularization term for numerical stability. $\left \|\cdot \right \|$ denotes the $\ell_2$ norm.
However, graph kernels used in KRR are computationally intensive and memory-consuming, limiting their scalability in large graphs.

\textbf{Distribution matching} directly aligns the class distributions of the original and condensed graphs~\cite{zhao2023dataset}. GCDM~\cite{liu2022graph} and CaT~\cite{liu_cat_2023} first introduce this strategy in GC by minimizing the discrepancy between class prototypes as:
\begin{equation}
\label{eq_DMeq}
\begin{aligned}
&\mathcal{L}_{DM}= {\mathop{\mathbb{E}}_{{\mathbf{\mathbf{\Theta}} } \sim \Phi}} \left[\left \| {\bf P}'{\bf{Z}}'- {\bf P}{\bf{Z}} \right \|^{2}\right ], \\
\end{aligned}
\end{equation}
where ${\bf P}'\in \mathbb{R}^{c\times N'}$ and ${\bf P}\in \mathbb{R}^{c\times N}$ are linear aggregation matrices to construct $c$ class-level features, i.e., class prototypes, for $\mathcal{S}$ and $\mathcal{T}$, respectively. 
Specifically, ${\bf P}'_{i,j} =\frac{1}{\left | C'_i \right | }$ if ${\bf y}'_j=i$, and ${\bf P}'_{i,j} =0$ otherwise. ${\bf P}_{i,j} =\frac{1}{\left | C_i \right | }$ if ${\bf y}_j=i$, and ${\bf P}_{i,j} =0$ otherwise. Here, ${\left | C'_i \right | }$ and ${\left | C_i \right | }$ represent the sizes of class $i$ in $\mathcal{S}$ and $\mathcal{T}$.
By eliminating the need to calculate model gradients and parameters, distribution matching achieves an efficient and flexible condensation process, making it prevalent in recent GC studies, e.g., eigenbasis matching~\cite{liu_graph_2023} and pre-trained model-based distribution matching~\cite{xiao2024disentangled,xiao2024simple}.

\textbf{Analysis.} Despite variations in format, existing optimization strategies can be uniformly formulated in a matching paradigm as:
\begin{equation}
\label{eq_def}
\arg\min_{\mathcal{S}}\mathcal{D}\left[ g\left(f\left({\mathcal{S}}\right), {\bf Y'}\right), g\left(f\left({\mathcal{T}}\right), {\bf Y}\right)\right],
\end{equation}
where $f$ is the relay model.
The function $g$ measures the matching objective for distance function $\mathcal{D}$, which varies in format for diverse optimization strategies as shown in Table \ref{tab_GCcom}.
Besides adopting the same framework, these three optimization strategies are inherently interconnected. Inspired by the study in CV~\cite{yu2023dataset}, we investigate their relationship and put forward three propositions as follows:
\begin{proposition}
\label{prop1}
The performance matching objective is equivalent to the optimal parameter matching objective.
\end{proposition}
\begin{proposition}
\label{prop2}
The distribution matching objective represents a simplified formulation of the performance matching, omitting feature correlation considerations.
\end{proposition}
\begin{proposition}
\label{prop3}
The distribution matching objective with a feature correlation constraint provides an upper bound for the parameter matching objective.
\end{proposition}

The proofs of propositions are deferred to the Appendix \ref{app_pro1}, \ref{app_pro2} and \ref{app_pro3}, respectively. 

\begin{remark}
In light of Propositions 1-3, various optimization strategies converge to class-level feature matching between the original and condensed graphs, i.e., distribution matching in Eq. (\ref{eq_DMeq}).
\end{remark}

However, on top of the efficiency issue discussed earlier, this matching paradigm only emphasizes on simulating the class distribution of the original graph by assigning a unified objective for all condensed nodes in the same class, resulting in the coarse-grained optimization target for each node.
To mitigate these issues, in the following subsection, we investigate the distribution matching objective and refine the feature matching from the class-to-class paradigm into the class-to-node paradigm, thereby providing each condensed node with an explicit and efficient optimization target.

\subsection{Simplifying Distribution Matching}

For the sake of simplicity, we first follow existing GC methods~\cite{jin2022graph,liu2022graph,xiao2024simple,gao_graph_2023,liu_cat_2023} to specify the relay model $f$ as widely used SGC~\cite{wu2019simplifying}, which decouples the propagation layer and transformation layer in GNN for efficient graph encoding as:
\begin{equation}
\label{eq_SGC}
\begin{aligned}
&{\bf{Z}}' = {\bf{H}}'{\mathbf{\Theta} } = {\hat{\bf{A}}}'^{K}{\bf{X}}'{\mathbf{\Theta} }, \\ 
&{\bf{Z}} = {\bf{H}}{\mathbf{\Theta} }= {\hat{\bf{A}}}^{K}{\bf{X}}{\mathbf{\Theta} },
\end{aligned}
\end{equation}
where ${\mathbf{\Theta} }$ is the learnable weight matrix to transform the $K$-th order propagated features ${\bf{H}}'$ and ${\bf{H}}$. 
${\hat{\bf{A}}'} = {\tilde{\bf{D}}}'^{\frac{1}{2}}{\tilde{\bf{A}}}'{\tilde{\bf{D}}}'^{\frac{1}{2}}$ and ${\hat{\bf{A}}}={\tilde{\bf{D}}}^{\frac{1}{2}}{\tilde{\bf{A}}}{\tilde{\bf{D}}}^{\frac{1}{2}}$ represent the symmetric normalized adjacency matrices, where ${\tilde{\bf{A}}'}$ and ${\tilde{\bf{A}}}$ are adjacency matrices with self-loops.  ${\tilde{\bf{D}}'}$ and ${\tilde{\bf{D}}}$ are degree matrices for ${\tilde{\bf{A}}'}$ and ${\tilde{\bf{A}}}$, respectively.
Consequently, the distribution matching objective in Eq. (\ref{eq_DMeq}) is formulated as:
\begin{equation}
\label{eq_sgcDM}
\begin{aligned}
\arg\min_{\mathcal{S}} {\mathop{\mathbb{E}}_{{\mathbf{\Theta} } \sim \Phi}} \left[\left \| {\bf P}'{\bf{H}}'{\mathbf{\Theta}}- {\bf P}{\bf{H}}{\mathbf{\Theta}} \right \|^{2} \right]. \\
\end{aligned}
\end{equation}
To facilitate the \textit{class-to-node distribution matching paradigm}, we enhance the aggregation matrix ${\bf{P}}$ with two objectives: 
\begin{itemize}[leftmargin=*]
\item The number of aggregated features is expanded from $c$ to $N'$, ensuring that each aggregated feature in the original graph corresponds to a distinct condensed node;
\item Aggregations are performed within classes to preserve the class semantics of aggregated features.
\end{itemize}

Consequently, it is expected that $|C_i|$ original nodes in class $i$ will be aggregated into $|C'_i|$ features to match with condensed nodes in $C'_i$.
This aggregation procedure analogizes to a class partition problem, defined as follows:
\begin{definition}
{{\bf Class partition.}} The class partition divides $|C_i|$ nodes in class $i$ into $|C'_i|$ non-overlapping sub-classes $\{S^i_1, ..., S^i_{|C'_i|}\}$, with each sub-class characterized by a centroid aggregated by constituent nodes.  The node mapping function of this partition is defined as $\pi^i:\{1, ..., |C_i|\} \to \{1, ..., |C'_i|\}$ and the class-wise aggregation matrix ${\bf{R}}^i \in \mathbb{R}^{|C'_i|\times |C_i|}$ is formulated as:
\begin{equation}
\label{eq_agg}
{{\bf{R}}}^{i}_{j,k} = \begin{cases}
\frac{1}{\left | S^{i}_j \right | }    & {\rm{if}} \ {\pi^i}(k)=j \\
 0  &       \rm{otherwise}
\end{cases},
\end{equation}
\end{definition}
\noindent where ${{\bf{R}}}^{i}_{j,k}$ denotes the aggregation weight for node $k$ in class $i$, and $\pi^i(k)$ indicates the subclass index for node $k$. 
$\left | S^{i}_j \right |$ represents the size of sub-class $S^{i}_j$.

With the defined class partition, the aggregation matrix ${\bf{P}}\in \mathbb{R}^{c\times N}$ can be updated to $\hat{\bf{P}}\in \mathbb{R}^{N'\times N}$, which is constructed by organizing all class-wise aggregation matrices along the diagonal as $\hat{\bf{P}}= \text{diag}({\bf{R}}^1, ..., {\bf{R}}^c)$.
Assuming that condensed nodes are organized in ascending order based on their labels, ${\bf{P}'}$ degrades into the identity matrix $\bf I$, and Eq. (\ref{eq_sgcDM}) can be reformulated as:
\begin{equation}
\label{eq_simDM1}
\begin{aligned}
&\arg\min_{\mathcal{S}, \hat{\bf P}}{\mathop{\mathbb{E}}_{{\mathbf{\Theta}} \sim \Phi}} \left[ \left \| {\bf{H}}'{\mathbf{\Theta}}- \hat{\bf P}{\bf{H}}{\mathbf{\Theta}} \right \|^{2} \right],  \\
\end{aligned}
\end{equation}
which is upper-bounded by:
\begin{equation}
\label{eq_simDM2}
\begin{aligned}
&{\mathop{\mathbb{E}}_{{\mathbf{\Theta}} \sim \Phi}} \left[ \left \| {\bf{H}}'{\mathbf{\Theta}}- \hat{\bf P}{\bf{H}}{\mathbf{\Theta}} \right \|^{2} \right] \leq {\mathop{\mathbb{E}}_{{\mathbf{\Theta} } \sim \Phi}} \left[ \left \| {\bf{H}}'- \hat{\bf P}{\bf{H}} \right \|^{2} \left \| {\mathbf{\Theta}}  \right \|^{2}\right].  \\
\end{aligned}
\end{equation}

Given that ${\mathbf{\Theta}}$ is independent to $\mathcal{S}$ and $\hat{\bf P}$, we can minimise the upper-bound by achieving: 
\begin{equation}
\label{eq_finDM}
\begin{aligned}
\arg\min_{\mathcal{S}, \hat{\bf P}}\left \| {\bf{H}}'- \hat{\bf P}{\bf{H}}\right \|^{2}. \\
\end{aligned}
\end{equation}

This objective can be further simplified under the graphless GC variant, where $\mathcal{S}=\{{\bf I}, {\bf X'}\}$ and ${\bf{H}}'={\hat{\bf{A}}}'^{K}{\bf{X}}'={{\bf{I}}}^{K}{\bf{X}}'={\bf{X}}'$. Consequently, the objective is reformulated to:
\begin{equation}
\label{eq_graphlessDM}
\begin{aligned}
\arg\min_{{\bf{X}}', \hat{\bf P}}\left \| {\bf{X}}'- \hat{\bf P}{\bf{H}}\right \|^{2}. \\
\end{aligned}
\end{equation}

\begin{remark}
\label{remark2}
The objective in Eq. (\ref{eq_graphlessDM}) indicates that the condensed node feature ${\bf{X}'}$ in graphless GC can be obtained by performing class partition on the propagated features ${\bf{H}}$, eliminating the gradient-based distribution matching optimization. This ${\bf{X}'}$ can be directly used to train GNNs with the identity matrix ${\bf{I}}$.
\end{remark}
Specifically, we can efficiently obtain the solution by applying any Expectation-Maximization (EM) based clustering algorithms (e.g., K-means) to each class, iteratively updating the cluster centroid ${\bf{X}'}$ and the aggregation matrix $\hat{\bf{P}}$ until convergence.
Note that another intuitive case arises when the size of the condensed graph matches the number of classes, i.e., $N'=c$.
In this scenario, ${\bf{P}'}=\bf I$ and the conventional distribution matching objective in Eq. (\ref{eq_sgcDM}) degrades to calculating the class prototypes with the pre-defined ${\bf{P}}$, which is consistent with Eq. (\ref{eq_graphlessDM}).

\begin{table}
\setlength{\abovecaptionskip}{0.1cm}
\centering
\caption{The comparison of simplified distribution matching with the conventional objective. Speedup ratios compared to GCDM-X are indicated in brackets ($r$: condensation ratio, model: GCN).}
\label{tab_pre}
\resizebox{\linewidth}{!}{
\begin{tabular}{lccccccc}
\toprule
\multirow{2}{*}{Dataset} & \multirow{2}{*}{$r$} & \multicolumn{3}{c}{Accuracy (\%)}                                                                                                               & \multicolumn{3}{c}{Condensation time (s)}                                                      \\ \cmidrule(l{3pt}r{4pt}){3-5} \cmidrule(l{4pt}r{3pt}){6-8}
                   &      & {GCDM} & {GCDM-X} & {SimDM} & {GCDM} & {GCDM-X} & {SimDM} \\ \midrule
Cora     &   2.60\%             & 77.2±0.4                                & 81.4±0.1                                 & 80.1±0.7                  & 27.3                     & 24.0                       & 0.3 (80$\times$)                 \\
Citeseer  &     1.80\%          & 69.5±1.1                              & 71.9±0.5                                  & 70.9±0.6                  & 62.0                     & 51.0                       & 0.7 (73$\times$)                 \\
Ogbn-arxiv &    0.25\%          & 59.6±0.4                             & 61.2±0.1                                   & 65.1±0.4                  & 690.0                    & 469.0                      & 7.1 (66$\times$)                 \\
Flickr  &       0.50\%           & 46.8±0.1                             & 45.6±0.1                                 & 45.8±0.2                  & 209.5                    & 86.2                       & 2.1 (41$\times$)                 \\
Reddit  &      0.10\%           & 89.7±0.2                             & 87.2±0.1                                    & 90.6±0.1                  & 921.9                    & 534.8                      & 8.2 (65$\times$)                 \\ \bottomrule
\end{tabular}
}
\end{table}

To validate the simplified objective, we compare the condensed graph $\{{\bf{I}}, {\bf{X}'}\}$, derived by objective in Eq. (\ref{eq_graphlessDM}) and termed SimDM, with the conventional distribution matching-based methods (i.e., GCDM~\cite{liu2022graph} and its graphless variation GCDM-X).
The test accuracy and condensation time are detailed in Table \ref{tab_pre} and experimental setting are deferred to Section \ref{sec_exp}. SimDM significantly excels in condensation time while maintaining comparable accuracy, confirming the effectiveness of our class partition-based objective.

%% file: 3met.tex
\section{Class-partitioned Graph Condensation (CGC)}
Despite the potential to accelerate the condensation procedure, the class partition in SimDM neglects the data quality of the original graphs, a critical aspect in addressing data-centric challenges in GC.
To establish a comprehensive data processing framework, we further enhance SimDM by incorporating data assessment, augmentation, and graph generation modules. 
Consequently, we introduce a novel GC framework, as illustrated in Figure \ref{fig_pip}, where all five involved modules are training-free, facilitating an efficient and robust condensation process.

\subsection{Feature Propagation}
To eliminate the gradient calculation in the condensation procedure, the non-parametric feature propagation module is deployed to smooth the node features $\bf{X}$ according to the original graph structure $\bf{A}$ and generate the node embeddings. 
The propagation method is replaceable with a variety of choices with diverse characteristics, e.g., SGC~\cite{wu2019simplifying}, personalized PageRank (PPR)~\cite{page1999pagerank} and SAGE~\cite{hamilton2017inductive}, etc. 
Without loss of generality, we follow existing GC methods~\cite{jin2022graph,liu2022graph,xiao2024simple,gao_graph_2023,liu_cat_2023} and adopt the propagation method in SGC to generate the node embeddings:
\begin{equation}
\label{eq_FP}
{\bf{H}}^{(l)} = {\hat{\bf{A}}}^{l}{\bf{X}},
\end{equation}
where $0\leq l \leq K$, and the embeddings in the last layer are denoted as ${\bf{H}} = {\bf{H}}^{(K)}$ for the final class partition. 

\subsection{Data Assessment}
For the sake of robust class representation, we assess the node embeddings prior to the class partition process.
Inspired by \cite{zhao2024graphany,zhu2020simple}, we incorporate node embeddings at various propagation depths and utilize a simple linear layer $\bf{W}$ as the classifier for label prediction.
\begin{equation}
\label{eq_DA}
\begin{aligned}
&\hat{\bf{Y}} = {\bf{T}}{\bf{W}} = \frac{1}{K+1}\sum_{l=0}^{K}{\bf{H}}^{(l)}{\bf{W}}.\\
\end{aligned}
\end{equation}
Subsequently, the MSE loss is employed for optimization:
\begin{equation}
\label{eq_DA2}
\begin{aligned}
&\arg\min_{{\bf{W}}} \left \|{\bf{Y}}-\hat{\bf{Y}} \right \|^{2}.
\end{aligned}
\end{equation}
This loss facilitates an efficient and precise closed-form solution, expressed as ${\hat{\bf{W}}}={\bf{T}^{+}}{\bf{Y}}$, where  ${\bf{T}^{+}}$ denotes the pseudo inverse of ${\bf{T}}$. 
Afterwards, the linear classifier is utilised to evaluate node embeddings ${\bf{H}}^{(l)}$ for $0\leq l \leq K$, and their confidence scores w.r.t the ground-truth are recorded as ${\bf {r}}\in \mathbb{R}^{N(K+1)}$.
Additionally, we calculate the prediction errors for each class of ${\bf{H}}$ and represent these class prediction errors as ${\bf e}=[e_1, ..., e_c]$.
The confidence scores ${\bf {r}}$ reflect the reliability of node embeddings in class representation, while the class prediction errors ${\bf e}$ highlight the difficulty associated with each class. In subsequent modules, these metrics will inform the data augmentation strategies and calibration of condensed features, aiming to enhance the quality of the condensed graph.

\begin{figure}[t]
\setlength{\abovecaptionskip}{4pt}
    \centering
    \includegraphics[width=0.95\linewidth]{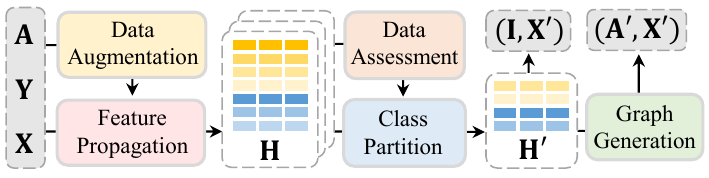}
    \caption{The pipeline of CGC and CGC-X.}
    \label{fig_pip}
\end{figure}

\subsection{Data Augmentation}
Although the non-overlapping class partition in SimDM simplifies optimization, it reduces the number of original nodes matched by each condensed node, intensifying the dependency on the quantity of nodes involved in the partition process.
Concurrently, the time complexity of class partition scales proportionally with the increasing number of nodes.
To address these issues and achieve equilibrium, we propose augmenting underrepresented classes with additional node embeddings to ensure robust class representations.

Specifically, we utilize the node embeddings with smaller propagation depths as augmentations, i.e., ${\bf{H}}^{(l)}$, where $0\leq l< K$. It is important to note that only training set nodes are involved, with each node paired with its respective label. For clarity, we refer to these node embeddings collectively as $\mathbb{H}$.
The utility of these embeddings offers threefold benefits. 
Firstly, these embeddings are an intermediate product in the feature propagation process and incur no additional computational overhead. Secondly, they aggregate node features across various hops, exhibiting different levels of smoothness \cite{zhang2021node}. 
Lastly, different from data augmentation methods such as mixup, edge drop, or feature mask, these embeddings prevent the introduction of excessive noises during the condensation process. 

Subsequently, class prediction errors $\bf e$ are used as sampling weights to randomly select $p\%$ of the embeddings from $\mathbb{H}$, where $p$ serves as a hyper-parameter controlling the size of augmentation data. 
This sampling strategy prioritizes classes with higher errors, potentially enhancing the class representation quality. 

Consequently, the embeddings and labels of the sampled nodes are represented as $\mathbf{H}_{aug}$ and $\mathbf{y}_{aug}$, respectively, and the augmented data for the condensation process are defined as $\mathbf{H}_{cond} = [\mathbf{H}; \mathbf{H}_{aug}]$ and $\mathbf{y}_{cond} = [\mathbf{y}; \mathbf{y}_{aug}]$.

\subsection{Class Partition}
With enhanced label ${\bf{y}}_{cond}$, we perform class partition on ${\bf{H}}_{cond}$ to address the optimization problem in Eq. (\ref{eq_graphlessDM}).
Initially, nodes labeled in ${\bf{y}}_{cond}$ are categorised into $c$ classes $\{\hat C_1, \cdots, \hat C_c\}$ with corresponding embeddings $\{\hat{\bf{H}}_1, \cdots, \hat{\bf{H}}_c\}$.
Then, $\hat{\bf{H}}_i$ for class $i$ is partitioned into ${|C'_i|}$ sub-classes $\{S^i_1, ..., S^i_{|C'_i|}\}$ by applying any 
EM-based clustering method and the node mapping function is denoted by $\pi^i$.

Subsequently, rather than utilizing the class-wise aggregation matrix ${\bf{R}}^i$ as initially defined in Eq. (\ref{eq_agg}), we incorporate the confidence scores $\bf r$ for each aggregated node to calibrate the condensed node embeddings. Consequently, Eq. (\ref{eq_agg}) is updated as: 
\begin{equation}
\label{eq_agg2}
{{\bf{R}}}^{i}_{j,k} = \begin{cases}
\frac{{\bf r}_k}{\tau}   & {\rm{if}} \ {\pi^i}(k)=j \\
 0  &       \rm{otherwise}
\end{cases},
\end{equation}
where ${{\bf{R}}}^{i}_{j,k}$ denotes the aggregation weight for node $k$ in subclass $j$, and $\pi^i(k)$ indicates the subclass index for node $k$. ${\bf r}_k$ is the confidence score of node $k$, and $\tau$ represents the global temperature to control the sensitivity for confidence scores.
Following this, the row-normalized aggregation matrix ${\bf \hat{R}}^i=\text{norm}({\bf {R}}^i)$ is employed to produce aggregated embeddings as: ${\bf{H}}'_i=\hat{{\bf{R}}}^i\hat{\bf{H}}_i$, and the condensed node embeddings are constructed by: ${\bf{H}}'= [{\bf{H}}'_1; ...; {\bf{H}}'_c]$.

\begin{table*}[t]
\renewcommand{\arraystretch}{1.2}
\setlength{\abovecaptionskip}{1pt}
\center
\caption{{{The accuracy (\%) comparison between our methods (CGC and CGC-X) and baselines. OOM means out-of-memory. The {best} (bold) and {runner-up} (underlined) performances for $({\bf I ,X'})$ and $({\bf A',X'})$ are highlighted, respectively.}}}
\label{tab_acc}
\setlength{\tabcolsep}{3.5pt}
\resizebox{\textwidth}{!}{
\begin{tabular}{c|r|cccccc|ccccccccccc|c}
\Xhline{1.pt}
                                                                             & \multicolumn{1}{l|}{}                      & \multicolumn{6}{c|}{$({\bf I ,X'})$}                                                                   & \multicolumn{11}{c|}{$({\bf A',X'})$}                                                                                                                                                                                                    &                                                                           \\ \cline{3-19}
\multirow{-2}{*}{\begin{tabular}[c]{@{}l@{}}Dataset\end{tabular}} & \multicolumn{1}{c|}{\multirow{-2}{*}{$r$}} & GCond-X  & GCDM-X         & SNTK-X         & SFGC              & GEOM              & CGC-X             & VN       & A-CM     & GCond             & GCDM     & DosCond  & SGDD     & MSGC     & SNTK              & SimGC          & GCSR              & CGC               & \multirow{-2}{*}{\begin{tabular}[c]{@{}l@{}}Whole\\ Dataset\end{tabular}} \\ \hline
                                                                             & 1.30\%                                     & 75.9\scriptsize{±1.2} & 81.3\scriptsize{±0.4}       & {\ul 82.2\scriptsize{±0.3}} & 80.1\scriptsize{±0.4}          & 80.3\scriptsize{±1.1}          & \textbf{83.4\scriptsize{±0.3}} & 31.2\scriptsize{±0.2} & 74.6\scriptsize{±0.1} & 79.8\scriptsize{±1.3}          & 69.4\scriptsize{±1.3} & 80.5\scriptsize{±0.1} & 80.1\scriptsize{±0.7} & 80.0\scriptsize{±0.8} & {\ul 81.7\scriptsize{±0.7}}    & 80.8\scriptsize{±2.3}       & 79.9\scriptsize{±0.7}          & \textbf{82.7\scriptsize{±0.3}} &                                                                           \\
                                                                             & 2.60\%                                     & 75.7\scriptsize{±0.9} & 81.4\scriptsize{±0.1}       & {\ul 82.4\scriptsize{±0.5}} & 81.7\scriptsize{±0.5}          & 81.5\scriptsize{±0.8}          & \textbf{83.4\scriptsize{±0.4}} & 65.2\scriptsize{±0.6} & 72.8\scriptsize{±0.2} & 80.1\scriptsize{±0.6}          & 77.2\scriptsize{±0.4} & 80.1\scriptsize{±0.5} & 80.6\scriptsize{±0.8} & 81.0\scriptsize{±0.5} & {\ul 81.5\scriptsize{±0.7}}    & 80.9\scriptsize{±2.6}       & 80.6\scriptsize{±0.8}          & \textbf{82.3\scriptsize{±1.3}} &                                                                           \\
\multirow{-3}{*}{\begin{tabular}[c]{@{}l@{}}Cora\end{tabular}}     & 5.20\%                                     & 76.0\scriptsize{±0.9} & {\ul 82.5\scriptsize{±0.3}} & 82.1\scriptsize{±0.1}       & 81.6\scriptsize{±0.8}          & 82.2\scriptsize{±0.4}          & \textbf{82.8\scriptsize{±1.0}} & 70.6\scriptsize{±0.1} & 78.0\scriptsize{±0.3} & 79.3\scriptsize{±0.3}          & 79.4\scriptsize{±0.1} & 80.3\scriptsize{±0.4} & 80.4\scriptsize{±1.6} & 80.1\scriptsize{±0.4} & 81.3\scriptsize{±0.2}          & {\ul 82.1\scriptsize{±1.3}} & 81.2\scriptsize{±0.9}          & \textbf{82.5\scriptsize{±0.6}} & \multirow{-3}{*}{81.2\scriptsize{±0.2}}                                                \\ \hline
                                                                             & 0.90\%                                     & 71.4\scriptsize{±0.8} & 69.0\scriptsize{±0.5}       & 69.9\scriptsize{±0.4}       & {\ul 71.4\scriptsize{±0.5}}    & 71.1\scriptsize{±0.2}          & \textbf{72.1\scriptsize{±0.2}} & 52.2\scriptsize{±0.4} & 65.1\scriptsize{±0.1} & 70.5\scriptsize{±1.2}          & 62.0\scriptsize{±0.1} & 71.0\scriptsize{±0.2} & 69.5\scriptsize{±0.4} & 72.4\scriptsize{±0.5} & 66.4\scriptsize{±1.0}          & \textbf{ 73.8\scriptsize{±2.5}} & 70.2\scriptsize{±1.1}          & {\ul 72.5\scriptsize{±0.2}} &                                                                           \\
                                                                             & 1.80\%                                     & 69.8\scriptsize{±1.1} & 71.9\scriptsize{±0.5}       & 69.9\scriptsize{±0.5}       & {\ul 72.4\scriptsize{±0.4}}    & 71.3\scriptsize{±0.1}          & \textbf{72.6\scriptsize{±0.2}} & 59.0\scriptsize{±0.5} & 66.0\scriptsize{±0.2} & 70.6\scriptsize{±0.9}          & 69.5\scriptsize{±1.1} & 71.2\scriptsize{±0.2} & 70.2\scriptsize{±0.8} & 72.1\scriptsize{±0.3} & 68.4\scriptsize{±1.1}          & {\ul 72.2\scriptsize{±0.5}} & 71.7\scriptsize{±0.9}          & \textbf{72.4\scriptsize{±0.2}} &                                                                           \\
\multirow{-3}{*}{\begin{tabular}[c]{@{}l@{}}Citeseer\end{tabular}} & 3.60\%                                     & 69.4\scriptsize{±1.4} & \textbf{ 72.8\scriptsize{±0.6}} & 69.1\scriptsize{±0.4}       & 70.6\scriptsize{±0.7}          & {\ul 72.1\scriptsize{±1.0}}          & {71.4\scriptsize{±0.4}} & 65.3\scriptsize{±0.5} & 66.1\scriptsize{±0.2} & 69.8\scriptsize{±1.4}          & 69.8\scriptsize{±0.2} & 70.7\scriptsize{±0.1} & 70.3\scriptsize{±1.7} & 71.9\scriptsize{±0.9} & 69.8\scriptsize{±0.8}          & 71.1\scriptsize{±2.8}       & \textbf{74.0\scriptsize{±0.4}} & {\ul 72.0\scriptsize{±0.5}}    & \multirow{-3}{*}{71.7\scriptsize{±0.1}}                                                \\ \hline
                                                                             & 0.05\%                                     & 61.3\scriptsize{±0.5} & 61.0\scriptsize{±0.1}       & 63.9\scriptsize{±0.3}       & \textbf{ 65.5\scriptsize{±0.7}}    & {\ul 64.7\scriptsize{±0.4}}          & 64.0\scriptsize{±0.1} & 35.4\scriptsize{±0.3} & 58.0\scriptsize{±0.1} & 59.2\scriptsize{±1.1}          & 59.3\scriptsize{±0.3} & 62.1\scriptsize{±0.3} & 60.8\scriptsize{±1.3} & 61.5\scriptsize{±1.1} & \textbf{ 64.4\scriptsize{±0.2}}    & 63.6\scriptsize{±0.8}       & 60.6\scriptsize{±1.1}          & {\ul 64.1\scriptsize{±0.4}} &                                                                           \\
                                                                             & 0.25\%                                     & 64.2\scriptsize{±0.4} & 61.2\scriptsize{±0.1}       & 65.5\scriptsize{±0.1}       & 66.1\scriptsize{±0.4}          & \textbf{67.5\scriptsize{±0.3}} & {\ul 66.3\scriptsize{±0.3}}    & 43.5\scriptsize{±0.2} & 60.0\scriptsize{±0.3} & 63.2\scriptsize{±0.3}          & 59.6\scriptsize{±0.4} & 63.5\scriptsize{±0.1} & 65.8\scriptsize{±1.2} & 65.8\scriptsize{±0.3} & 65.1\scriptsize{±0.8}          & \textbf{66.4\scriptsize{±0.3}} & 65.4\scriptsize{±0.8}          & \textbf{66.4\scriptsize{±0.1}} &                                                                           \\
\multirow{-3}{*}{\begin{tabular}[c]{@{}l@{}}Arxiv\end{tabular}}    & 0.50\%                                     & 63.1\scriptsize{±0.5} & 62.5\scriptsize{±0.1}       & 65.7\scriptsize{±0.4}       & 66.8\scriptsize{±0.4}          & \textbf{67.6\scriptsize{±0.2}} & {\ul 67.0\scriptsize{±0.1}}    & 50.4\scriptsize{±0.1} & 61.0\scriptsize{±0.2} & 64.0\scriptsize{±0.4}          & 62.4\scriptsize{±0.1} & 63.7\scriptsize{±0.2} & 66.3\scriptsize{±0.7} & 66.0\scriptsize{±0.3} & 65.4\scriptsize{±0.5}          & {\ul 66.8\scriptsize{±0.4}} & 65.9\scriptsize{±0.6}          & \textbf{67.2\scriptsize{±0.4}} & \multirow{-3}{*}{71.4\scriptsize{±0.1}}                                                \\ \hline
                                                                             & 0.10\%                                     & 45.9\scriptsize{±0.1} & 46.0\scriptsize{±0.1}       & 46.6\scriptsize{±0.3}       & {\ul 46.6\scriptsize{±0.2}}    & 46.1\scriptsize{±0.5}          & \textbf{46.7\scriptsize{±0.2}} & 41.9\scriptsize{±0.2} & 42.2\scriptsize{±0.1} & 46.5\scriptsize{±0.4}          & 46.1\scriptsize{±0.1} & 46.0\scriptsize{±0.3} & 46.5\scriptsize{±0.1} & 46.2\scriptsize{±0.1} & \textbf{46.7\scriptsize{±0.1}} & 45.3\scriptsize{±0.7}       & 46.6\scriptsize{±0.3}          & \textbf{46.8\scriptsize{±0.0}} &                                                                           \\
                                                                             & 0.50\%                                     & 45.0\scriptsize{±0.2} & 45.6\scriptsize{±0.1}       & 46.7\scriptsize{±0.1}       & \textbf{47.0\scriptsize{±0.1}} & 46.2\scriptsize{±0.2}          & \textbf{47.0\scriptsize{±0.1}} & 44.5\scriptsize{±0.1} & 45.2\scriptsize{±0.3} & \textbf{47.1\scriptsize{±0.1}} & 46.8\scriptsize{±0.1} & 46.2\scriptsize{±0.2} & 46.4\scriptsize{±0.2} & 46.4\scriptsize{±0.1} & 46.8\scriptsize{±0.1}          & 45.6\scriptsize{±0.4}       & 46.6\scriptsize{±0.2}          & \textbf{47.1\scriptsize{±0.1}} &                                                                           \\
\multirow{-3}{*}{\begin{tabular}[c]{@{}l@{}}Flickr\end{tabular}}   & 1.00\%                                     & 45.0\scriptsize{±0.1} & 45.4\scriptsize{±0.3}       & 46.6\scriptsize{±0.2}       & \textbf{47.1\scriptsize{±0.1}} & 46.7\scriptsize{±0.1}          & {\ul 47.0\scriptsize{±0.1}}    & 44.6\scriptsize{±0.1} & 45.1\scriptsize{±0.1} & \textbf{47.1\scriptsize{±0.1}} & 46.7\scriptsize{±0.1} & 46.1\scriptsize{±0.1} & 46.3\scriptsize{±0.1} & 46.2\scriptsize{±0.1} & 46.5\scriptsize{±0.2}          & 43.8\scriptsize{±1.5}       & 46.8\scriptsize{±0.2}          & {\ul 47.0\scriptsize{±0.1}}    & \multirow{-3}{*}{47.2\scriptsize{±0.1}}                                                \\ \hline
                                                                             & 0.05\%                                     & 88.4\scriptsize{±0.4} & 86.5\scriptsize{±0.2}       & OOM            & 89.7\scriptsize{±0.2}          & {\ul 90.1\scriptsize{±0.2}}    & \textbf{90.3\scriptsize{±0.2}} & 40.9\scriptsize{±0.5} & 72.2\scriptsize{±1.2} & 88.0\scriptsize{±1.8}          & 89.3\scriptsize{±0.1} & 89.8\scriptsize{±0.1} & 90.3\scriptsize{±0.1} & 89.8\scriptsize{±0.2} & OOM               & 89.6\scriptsize{±0.6}       & {\ul 90.5\scriptsize{±0.2}}    & \textbf{90.6\scriptsize{±0.2}} &                                                                           \\
                                                                             & 0.10\%                                     & 89.3\scriptsize{±0.1} & 87.2\scriptsize{±0.1}       & OOM            & 90.0\scriptsize{±0.3}          & {\ul 90.4\scriptsize{±0.1}}    & \textbf{90.8\scriptsize{±0.0}} & 42.8\scriptsize{±0.8} & 73.5\scriptsize{±1.0} & 89.6\scriptsize{±0.7}          & 89.7\scriptsize{±0.2} & 90.5\scriptsize{±0.1} & 90.7\scriptsize{±0.1} & 89.2\scriptsize{±0.1} & OOM               & 90.6\scriptsize{±0.3}       & {\ul 91.2\scriptsize{±0.2}}    & \textbf{91.4\scriptsize{±0.1}} &                                                                           \\
\multirow{-3}{*}{\begin{tabular}[c]{@{}l@{}}Reddit\end{tabular}}   & 0.20\%                                     & 88.8\scriptsize{±0.4} & 88.8\scriptsize{±0.1}       & OOM            & 89.9\scriptsize{±0.4}          & {\ul 90.9\scriptsize{±0.1}}    & \textbf{91.4\scriptsize{±0.1}} & 47.4\scriptsize{±0.9} & 75.1\scriptsize{±1.3} & 90.1\scriptsize{±0.5}          & 90.2\scriptsize{±0.4} & 91.1\scriptsize{±0.1} & 91.3\scriptsize{±0.2} & 90.2\scriptsize{±0.1} & OOM               & 91.4\scriptsize{±0.2}       & \textbf{92.2\scriptsize{±0.1}} & {\ul 91.6\scriptsize{±0.2}}    & \multirow{-3}{*}{93.9\scriptsize{±0.0}}                                                \\ \hline
                                                                             & 0.025\%                                    & 64.5\scriptsize{±0.2} & 65.1\scriptsize{±0.1}       & OOM            & 66.2\scriptsize{±0.3}          & {\ul 67.7\scriptsize{±0.2}}    & \textbf{68.0\scriptsize{±0.0}} & 34.3\scriptsize{±0.8} & 58.8\scriptsize{±0.9} & 64.2\scriptsize{±0.1}          & 66.1\scriptsize{±0.1} & 62.3\scriptsize{±0.3} & 64.5\scriptsize{±0.1} & 64.9\scriptsize{±0.1} & OOM               & 63.7\scriptsize{±1.1}       & {\ul 66.5\scriptsize{±0.2}}    & \textbf{68.0\scriptsize{±0.1}} &                                                                           \\
                                                                             & 0.050\%                                    & 65.2\scriptsize{±0.3} & 66.8\scriptsize{±0.2}       & OOM            & 67.0\scriptsize{±0.2}          & {\ul 68.4\scriptsize{±0.3}}    & \textbf{68.9\scriptsize{±0.2}} & 35.1\scriptsize{±0.9} & 60.1\scriptsize{±0.6} & 64.7\scriptsize{±0.2}          & 67.4\scriptsize{±0.4} & 63.6\scriptsize{±0.1} & 64.9\scriptsize{±0.1} & 65.0\scriptsize{±0.3} & OOM               & 64.9\scriptsize{±1.2}       & {\ul 67.8\scriptsize{±0.3}}    & \textbf{68.9\scriptsize{±0.3}} &                                                                           \\
\multirow{-3}{*}{\begin{tabular}[c]{@{}l@{}}Products\end{tabular}} & 0.100\%                                    & 65.5\scriptsize{±0.2} & 67.2\scriptsize{±0.1}       & OOM            & {\ul 68.8\scriptsize{±0.3}}    & 68.7\scriptsize{±0.5}          & \textbf{69.0\scriptsize{±0.1}} & 37.4\scriptsize{±0.9} & 62.4\scriptsize{±0.9} & 65.0\scriptsize{±0.1}          & 68.4\scriptsize{±0.3} & 65.8\scriptsize{±0.2} & 64.7\scriptsize{±0.1} & 65.1\scriptsize{±0.2} & OOM               & 65.2\scriptsize{±1.4}       & {\ul 68.5\scriptsize{±0.3}}    & \textbf{69.1\scriptsize{±0.2}} & \multirow{-3}{*}{73.1\scriptsize{±0.0}}                                                \\   \Xhline{1.pt}
\end{tabular}
}
\end{table*}

\subsection{Graph Generation}
We follow existing GC methods to provide two parameterization methods for the condensed graph, including CGC with the graph generation and its graphless variant CGC-X. 

According to Eq. (\ref{eq_SGC}), we expect a symmetric encoding procedure between the original and condensed graphs. Therefore, our objective is to construct $\bf{A}'$ and ${\bf X}'$ satisfying ${\hat{\bf{A}}}'^{K}{\bf{X}'} = {\bf{H}'}$. 
To this end, we utilize the pre-defined graph structure and calculate ${\bf X}'$ in a close-formed solution. 

Specifically, we construct the condensed graph structure according to the condensed node embeddings~\cite{zhu2021deep} as follows:
\begin{equation}
\label{eq_adj}
{\bf{A}}_{i, j}' = \begin{cases}
 1 & \text{ if } \text{cos}({\bf{H}}_i', {\bf{H}}_j')> T \\
 0 & \text{otherwise}
\end{cases},
\end{equation}
where $\text{cos}(\cdot, \cdot)$ measures the cosine similarity and $T$ is the hyper-parameter for graph sparsification. 
To ensure that generated features change smoothly between connected nodes, we introduce the Dirichlet energy constraint~\cite{kalofolias2016learn} in feature reconstruction loss and quantify the smoothness of graph signals. The optimization objective for ${\bf X}'$ is formulated as:
\begin{equation}
\label{eq_feat}
\begin{aligned}
\mathcal{L} = \arg\min_{{\bf{X}}'}\left \| {\hat{\bf{A}}}'^{K}{\bf{X}}'- {\bf{H}}'\right \|^2+ \alpha\rm{tr}({\bf{X}}'^{\top}{\bf{L}}'{\bf{X}}'),
\end{aligned}
\end{equation}
where $\alpha$ balances the losses, and $\rm{tr}(\cdot)$ denotes the matrix trace. $\textbf{L}'= \textbf{D}' - \textbf{A}'$ is the Laplacian matrix, where $\textbf{D}'$ is the degree matrix.
Consequently, the closed-form solution for Eq. (\ref{eq_feat}) is presented as: 

\begin{proposition}
\label{prop4}
Assume an undirected condensed graph $\mathcal{S}=\{{\bf A'}, {\bf X'}\}$, the closed-form solution of Eq. (\ref{eq_feat}) takes the form: ${\bf X}'=({\bf Q}^{\top}{\bf Q}+\alpha {\bf L}')^{-1}{\bf Q}^{\top}{\bf H}'$, where ${\bf Q} = \hat{\bf{A}}'^{K}$.
\end{proposition}
The proof is deferred to the Appendix \ref{app_pro4}.
Although the closed-form solution involves an inverse operation, the target matrix is small (i.e., $N'$-by-$N'$) and can be efficiently calculated.

\noindent\textbf{A Graphless Variant.}
Based on Remark \ref{remark2}, the node embedding ${\bf{H}}'$ equivalents to the condensed node feature ${\bf{X}}'$ when utilizing ${\bf I}$ as the condensed graph structure and employing non-parametric feature propagation for graph encoding. 
Consequently, ${\bf{H}}'$ derived in the Class Partition module can directly serve as the condensed graph for CGC-X, i.e., $\mathcal{S}=\{{\bf I}, {\bf X'}\} = \{{\bf I}, {\bf H'}\} $.

\noindent\textbf{Comparison with Coarsening Methods.}
Similar to our simplified GC objective in Eq. (\ref{eq_graphlessDM}), coarsening methods develop the aggregation matrix to merge original nodes into super-nodes and transform the original graph structure into a smaller graph. 
However, these methods are implemented within an unsupervised paradigm, prioritizing the preservation of graph properties such as spectral~\cite{loukas2018spectrally} and cut~\cite{loukas2019graph} guarantees while disregarding label information~\cite{hashemi2024comprehensive}.
In contrast, our CGC framework synthesizes nodes and connections under a supervised paradigm, thereby enhancing the utility of downstream tasks.

The detailed {\bf{algorithm}} and {\bf{complexity analysis}} can be found in Appendix \ref{sec_appalg} and \ref{sec_apptime}, respectively.

%% file: 4exp.tex
\section{Experiments}
\label{sec_exp}

We design comprehensive experiments to validate the efficacy of our proposed methods and explore the following research questions:
\textbf{Q1}: Compared to the other graph reduction methods, can the condensed graph generated by CGC and CGC-X achieve better GNN performance?
\textbf{Q2}: Can the CGC and CGC-X condense the graph faster than other GC approaches?
\textbf{Q3}: Can the condensed graph generated by CGC and CGC-X generalize well to different GNN architectures? 
\textbf{Q4}: How do the different components, i.e., {{feature propagation}}, data augmentation, data assessment and class partition methods affect CGC and CGC-X?
\textbf{Q5}: How do the different hyper-parameters affect the CGC and CGC-X?

\subsection{Experimental Setup}
{\bf Datasets \& Baselines.}
We evaluate our proposed methods on four transductive datasets (Cora, Citeseer~\cite{DBLP:conf/iclr/KipfW17}, Ogbn-arxiv (Arxiv)~\cite{hu2020open} and Ogbn-products (Products)~\cite{hu2020open}), as well as two inductive datasets (Flickr and Reddit~\cite{DBLP:conf/iclr/ZengZSKP20}), all with public splits.
We compare {{15}} baselines, encompassing both graph coarsening and graph condensation methods with diverse optimization strategies:
(1) graph coarsening methods: Variation Neighborhoods (VN)~\cite{huang2021scaling,loukas2019graph} and A-ConvMatch (A-CM)~\cite{dickens2024graph};
(2) {{gradient matching-based GC methods: GCond, GCond-X~\cite{jin2022graph}, DosCond~\cite{jin2022condensing}, SGDD~\cite{yang_does_2023} and MSGC~\cite{gao2023multiple}; }}
(3) trajectory matching-based GC methods: SFGC~\cite{zheng_structure_free_2023}, GEOM~\cite{zhang2024navigating} and GCSR~\cite{liu2024graph}; 
(4) KRR-based GC methods: SNTK and SNTK-X~\cite{wang2023fast}; 
(5) distribution matching-based GC methods: GCDM, GCDM-X~\cite{liu2022graph} and SimGC~\cite{xiao2024simple}.
Notice that the suffix ``-X'' represents the graphless variant.
More details about the datasets and baselines are provided in Appendix \ref{sec_data} and \ref{sec_baseline}, respectively.

\begin{table*}[t]
\renewcommand{\arraystretch}{1.2}
\setlength{\abovecaptionskip}{1pt}
\center
\caption{{{The condensation time (seconds) comparison of different graph reduction methods. OOM denotes out-of-memory. $r$ is set as 2.60\%, 1.80\%, 0.25\%, 0.50\%, 0.10\% and 0.05\% for six datasets, respectively.}}}
\label{tab_time}
\resizebox{\textwidth}{!}{
\begin{tabular}{l|rrrrrrrrr rrrrrrrr}
\Xhline{1.pt}
Dataset  & \multicolumn{1}{l}{VN} & \multicolumn{1}{l}{A-CM} & \multicolumn{1}{l}{GCond} & \multicolumn{1}{l}{GCond-X} & \multicolumn{1}{l}{GCDM} & \multicolumn{1}{l}{GCDM-X} & \multicolumn{1}{l}{DosCond} & \multicolumn{1}{l}{SGDD} & \multicolumn{1}{l}{MSGC} & \multicolumn{1}{l}{SNTK} & \multicolumn{1}{l}{SNTK-X} & \multicolumn{1}{l}{SimGC} & \multicolumn{1}{l}{SFGC} & \multicolumn{1}{l}{GEOM} & \multicolumn{1}{l}{GCSR} & \multicolumn{1}{l}{CGC} & \multicolumn{1}{l}{CGC-X} \\ \hline
Cora     & 2.7                    & 7.6                             & 498.9                     & 80.2                        & 27.3                     & 24.0                       & 383.7                                               & 868.3                                            & 812.3                                            & 30.8                     & 20.3                       & 210.6                     & 2,524.4                  & 3,302.2                  & 850.4                    & 1.4                     & 0.4                       \\
Citeseer & 3.5                    & 141.4                           & 479.3                     & 70.6                        & 62.0                     & 51.0                       & 300.6                                               & 868.0                                            & 200.1                                            & 19.4                     & 18.7                       & 227.6                     & 3,877.0                  & 4,230.3                  & 488.6                    & 1.5                     & 0.8                       \\
Arxiv    & 412.6                  & 179.9                           & 14,876.0                  & 13,516.7                    & 690.0                    & 469.0                      & 12,483.0                                            & 22,155.8                                         & 14,723.1                                         & 9,353.6                  & 9,079.9                    & 252.4                     & 86,288.5                 & 108,962.5                & 5,283.5                  & 8.8                     & 7.9                       \\
Flickr   & 243.8                  & 160.1                           & 1,338.3                   & 1,054.7                     & 209.5                    & 86.2                       & 1,060.3                                             & 16,561.4                                         & 819.0                                            & 562.6                    & 510.0                      & 345.8                     & 52,513.0                 & 54,601.9                 & 1,202.6                  & 7.5                     & 6.8                       \\
Reddit   & 464.1                  & 192.1                           & 21,426.5                  & 22,154.3                    & 921.9                    & 534.8                      & 18,741.6                                            & 53,192.1                                         & 7,128.6                                          & \multicolumn{1}{r}{OOM}  & \multicolumn{1}{r}{OOM}    & 475.8                     & 246,997.1                & 248,837.4                & 1,747.3                  & 18.5                    & 17.0                        \\
Products & 59,074.1               & 6,254.6                         & 26,789.3                  & 24,145.4                    & 7,897.9                  & 5,790.6                    & 25,496.3                                            & 28,843.9                                         & 12,142.3                                         & \multicolumn{1}{r}{OOM}  & \multicolumn{1}{r}{OOM}    & 3,824.3                   & 283,068.2                & 284,941.5                & 18,901.30                & 31.5                    & 25.8                      \\ \Xhline{1.pt}
\end{tabular}
}
\end{table*}

\begin{figure*}[t]
\centering
\begin{minipage}[t]{0.33\linewidth}
\centering
\includegraphics[width=\linewidth]{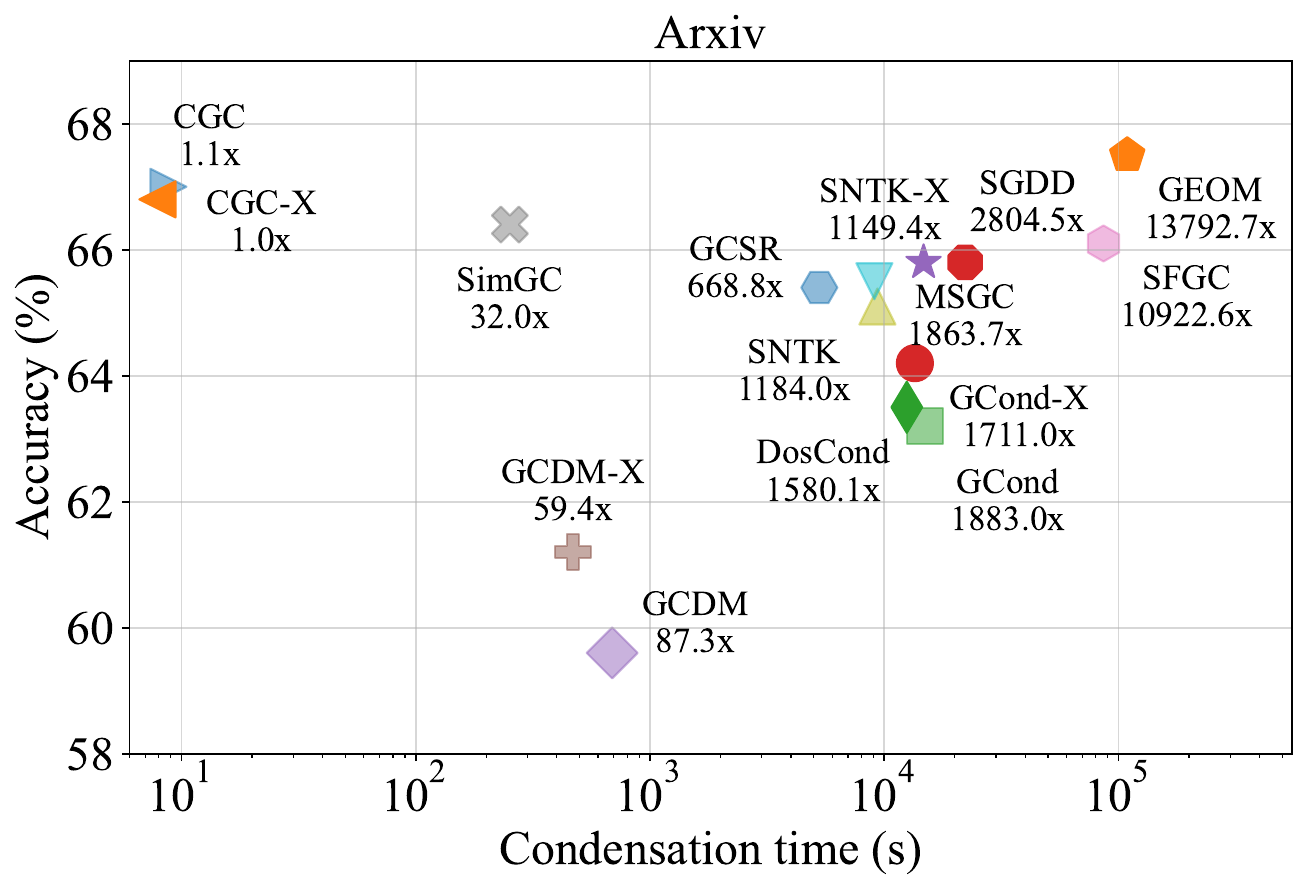}
\end{minipage}
\hspace{12mm}
\begin{minipage}[t]{0.33\linewidth}
\centering
\includegraphics[width=\linewidth]{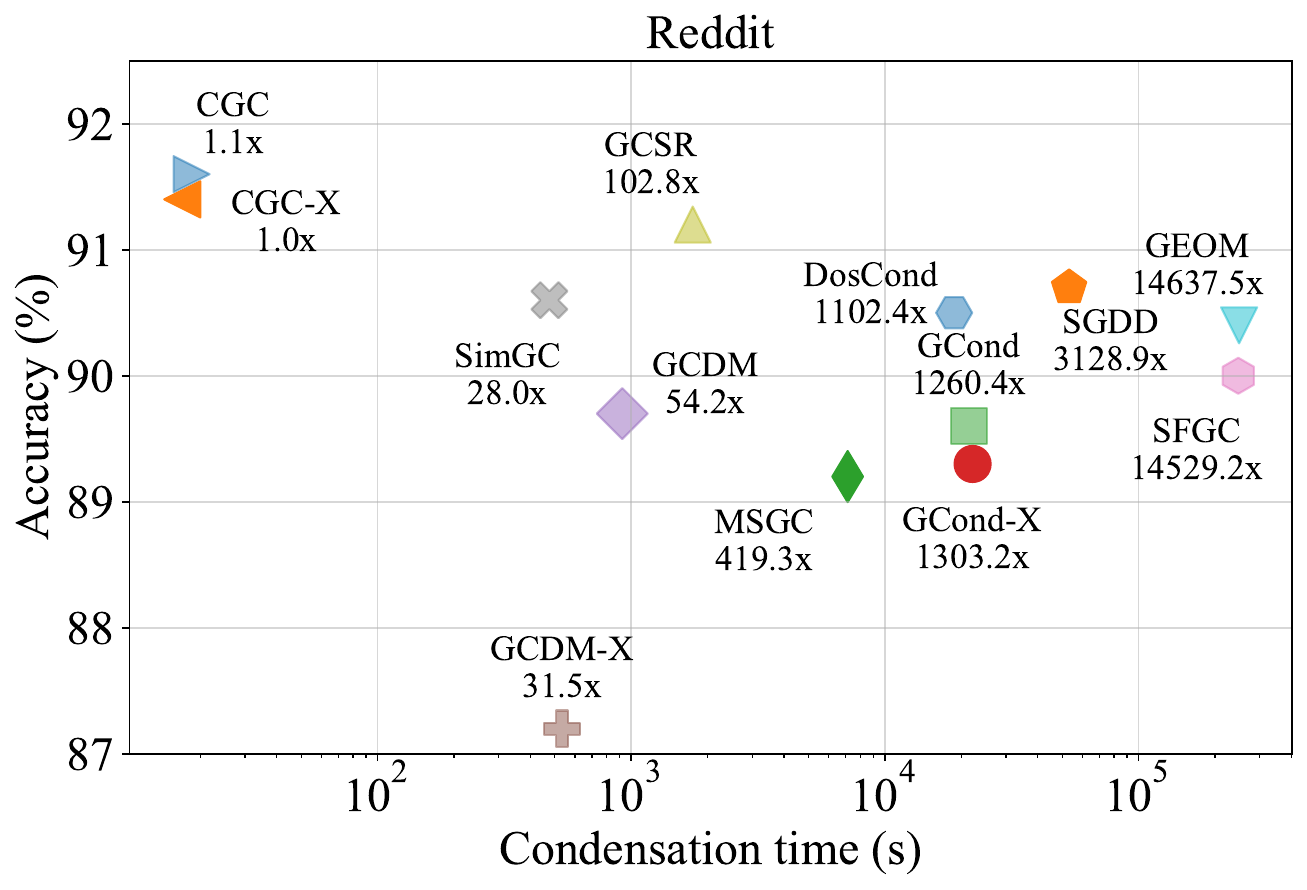}
\end{minipage}
\setlength{\abovecaptionskip}{1pt}
\caption{{{The accuracy and condensation time comparison of GC methods on Arxiv ($r=0.25\%$) and Reddit ($r=0.10\%$). SNTK and SNTK-X are out-of-memory on Reddit dataset.}}}
\label{fig_time}
\end{figure*}

\noindent{\bf Implementations.}
Following GCond~\cite{jin2022graph}, we evaluate three condensation ratios ($r=N'/N$) for each dataset.
In the transductive setting, $N$ represents the original graph size, while in the inductive setting, $N$ indicates the sub-graph size observed in the training stage.
Two-layer GNNs with 256 hidden units are used for evaluation.
We adopt the propagation method in SGC for feature propagation and spectral clustering~\cite{fettal2023scalable} with acceleration implementation (i.e., FAISS~\cite{douze2024faiss}) for class partition (refer to Section \ref{sec_ablation} for results of alternative class partition methods).
For reproducibility, other detailed implementations, hyper-parameters and computing infrastructure are summarized in Appendix \ref{sec_imple}.

\subsection{Effectiveness Comparison (Q1)}

For the sake of fairness, we compare CGC-X and CGC with graphless and graph-generated GC baselines separately. 
The condensed graphs generated by GC methods are evaluated to train a 2-layer GCN and the test accuracies with standard deviation are reported in Table \ref{tab_acc}.
In the table, ``Whole Dataset'' refers to the GCN performance which is trained on the original graph and we make the following observations.
Firstly, CGC-X and CGC consistently outperform other baselines across all datasets. While GCSR achieves the best performance on Citeseer and Reddit under the largest condensation ratio, the performance gap between CGC and GCSR remains small. Moreover, CGC significantly outperforms GCSR across other condensation ratios. Notably, on the Arxiv dataset, CGC demonstrates substantial improvement over GCSR, highlighting the superiority of our proposed method.
Furthermore, our proposed method can effectively mitigate the label sparsity issue in GC. On two datasets with sparse labels, i.e., Cora and Citeseer, CGC and CGC-X consistently achieve superior performances. This is contributed to the data augmentation module which can introduce reliable nodes for precise class distribution representation in the GC procedure.

\subsection{Efficiency Comparison (Q2)}
\label{sec_time}

We now report the condensation time of the proposed method and baselines on different datasets. For each method, we repeated the experiments 5 times and the averaged condensation time is reported in Table \ref{tab_time}. 
Instead of adopting the Siamese network architecture as other existing GC methods, CGC and CGC-X eliminate the gradient-based optimization and achieve extremely efficient condensation procedures. All datasets can be condensed within 1 minute, which is multiple orders of magnitude improvement compared to other GC methods.
In contrast, graph coarsening methods use a hierarchical node aggregation paradigm to iteratively reduce the graph size, which leads to high latency on large-scale datasets.
To facilitate a clearer comparison of the methods, Figure \ref{fig_time} shows the accuracy of GC methods over their condensation time on the transductive dataset (Arxiv) and inductive dataset (Reddit). The relative condensation time to the fastest CGC-X is marked. 
Our proposed CGC achieves the highest test accuracy, and CGC-X is 32.0$\times$ and 28.0$\times$ faster than the most efficient baseline SimGC.

\subsection{Generalizability Comparison (Q3)}

\begin{table}[t]
\renewcommand{\arraystretch}{1.2}
\setlength{\abovecaptionskip}{1pt}
\center
\caption{{{The generalizability of GC methods with graph generation on Arxiv ($r=0.25\%$). AVG indicates the average value.}}}
\label{tab_gene}
\resizebox{0.95\linewidth}{!}{
\begin{tabular}{l|rrrrrr|r}
    \Xhline{1.pt}
 Method                          & \multicolumn{1}{l}{SGC}      & \multicolumn{1}{l}{GCN}      & \multicolumn{1}{l}{SAGE}     & \multicolumn{1}{l}{APPNP}    & \multicolumn{1}{l}{Cheby}    & \multicolumn{1}{l|}{GAT}     & \multicolumn{1}{l}{AVG}      \\ \hline
                                                                             GCond                           & 63.7                         & 63.2                         & 62.6                         & 63.4                         & 54.9                         & 60.0                         & 61.3                         \\
                                                                             GCDM                            & 61.2                         & 59.6                         & 61.1                         & 62.8                         & 55.4                         & 61.2                         & 60.2                         \\
                                                                             DosCond & 63.3 & 63.5 & 62.1 & 63.5 & 55.1 & 60.4 & 61.3 \\
                                                                             SGDD    & 64.3 & 65.8 & 63.4 & 63.3 & 55.9 & 61.4 & 62.4 \\
                                                                             MSGC    & 64.6 & 65.8 & 60.2 & 62.2 & 55.1 & 61.2 & 61.5 \\
                                                                             SNTK                            & 62.7                         & 65.1                         & 62.9                         & 62.6                         & 55.1                         & 61.8                         & 61.7                         \\
                                                                             SimGC                           & 64.3                         & \textbf{66.4}                         & 60.4                         & 61.5                         & 54.7                         & 61.1                         & 61.4                         \\
                                                                             GCSR                            & \textbf{65.6}                & 65.4                         & 65.4                         & \textbf{64.4}                & 58.9                         & 63.5                         & 63.9                         \\
CGC                             & 64.9                         & \textbf{66.4}                & \textbf{65.7}                & 63.7                         & \textbf{60.5}                & \textbf{65.0}                & \textbf{64.4}                \\      \Xhline{1.pt}
\end{tabular}
}
\end{table}

To compare the generalizability across different GNN architectures, we assess the performance of GC methods under different GNN models, including GCN, SGC, SAGE~\cite{hamilton2017inductive}, APPNP~\cite{gasteiger_predict_2019}, Cheby~\cite{defferrard2016convolutional} and GAT~\cite{DBLP:conf/iclr/VelickovicCCRLB18}. The detailed accuracies of graph-generated GC methods on Arxiv dataset are shown in Table \ref{tab_gene}.
The results on other datasets and graphless GC methods can be found in Appendix \ref{sec_appgenedata} and \ref{sec_appgene}, respectively. 
{{We could observe that GCN and SGC outperform other GNNs because they employ the same propagation method as the feature propagation in the condensation procedure.}}
In addition, CGC achieves a significant improvement over other compared baselines.
This indicates the effectiveness of the pre-defined graph structure in our proposed method, which captures the relationship among aggregated features and encourages smoothness among connected condensed nodes.

\begin{table}[t]
    \renewcommand{\arraystretch}{1.2}
    \setlength{\abovecaptionskip}{1pt}
    \center
    \caption{{{The ablation study of CGC-X. $r$ is set as 2.60\%, 1.80\%, 0.25\%, 0.50\% and 0.10\% for five datasets, respectively.}}}
    \label{tab_abl}
    \resizebox{\linewidth}{!}{
    \begin{tabular}{l|cc ccc}
    \Xhline{1.pt}
    Method                              & Cora              & Citeseer          & Arxiv                            & Flickr                           & Reddit                           \\ \hline
    SimDM       & 80.1\scriptsize{±0.7}          & 70.9\scriptsize{±0.6}          & 65.1\scriptsize{±0.4} & 45.8\scriptsize{±0.2} & 90.6\scriptsize{±0.1} \\ \hline
    
    CGC-X w/o AUG                       & 80.5\scriptsize{±0.5}          & 71.1\scriptsize{±0.3}          & 65.3\scriptsize{±0.1}                         & 46.6\scriptsize{±0.2}                         & 90.7\scriptsize{±0.1}                         \\
    
    CGC-X w/o CAL                       & 81.5\scriptsize{±0.4}          & 71.5\scriptsize{±0.4}          & 65.8\scriptsize{±0.2}                         & 46.2\scriptsize{±0.2}                         & 90.6\scriptsize{±0.1}                         \\
    
    CGC-X w K-means                      & 83.1\scriptsize{±0.7}          & 72.2\scriptsize{±0.3}          & 66.4\scriptsize{±0.1}                         & 46.9\scriptsize{±0.1}                         & 90.7\scriptsize{±0.1}                         \\
    
    CGC-X w RAN & 79.7\scriptsize{±0.5}          & 70.0\scriptsize{±0.1}          & 62.0\scriptsize{±0.1} & 45.5\scriptsize{±0.2} & 88.9\scriptsize{±0.1} \\ \hline
    
    CGC-X                               & \textbf{83.4\scriptsize{±0.4}} & \textbf{72.6\scriptsize{±0.2}} & \textbf{66.3\scriptsize{±0.1}}                & \textbf{47.0\scriptsize{±0.1}}                & \textbf{90.8\scriptsize{±0.0}}                \\ \Xhline{1.pt}
\end{tabular}}
\end{table}

\subsection{Ablation Study (Q4)}
\label{sec_ablation}

\noindent{\bf{Data Augmentation.}} 
To validate the impact of data augmentation, CGC and CGC-X are evaluated by disabling the augmentation component (termed ``w/o AUG'') and results are shown in Table \ref{tab_abl} and Appendix \ref{sec_appablation}, respectively. {{Due to the data augmentation, our proposed methods can introduce more labeled nodes in GC and facilitate a more precise class representation, leading to significant improvement on label-sparse datasets.}}

\noindent{\bf{Data Assessment.}} 
We compare the results of CGC and CGC-X with those obtained by removing the calibration of the condensed node embeddings (referred to as "w/o CAL"). When the class-wise aggregation matrix is replaced by Eq. (\ref{eq_agg}), the confidence score for each node is disregarded, leading to equal aggregation of all nodes within the clusters. This modification results in a performance drop across all datasets and verifies the effectiveness of our data assessment module.

\noindent{\bf{Partition Method.}} {{In addition to spectral clustering, we also tested K-means (``w K-means") and random partition (``w RAN"), with results presented in Table \ref{tab_abl} and Appendix \ref{sec_appablation}. The similar performance levels of K-means indicate that our proposed method is insensitive to the choice of partition method. However, random partition fails to achieve the optimal solution, highlighting the effectiveness of our optimization objective in Eq. (\ref{eq_graphlessDM}).}}

{{For the impact of propagation method and propagation depth, please refer to Appendix \ref{sec_apppropmethod} and \ref{sec_apppropdep} for more details.}}

\subsection{Hyper-parameter Sensitivity Analysis (Q5)}
\label{sec_apphyper}
Due to the training-free nature of our proposed methods, the number of hyper-parameters is significantly reduced. CGC contains four hyper-parameters: the constraint weight $\alpha$, adjacency matrix threshold $T$, temperature $\tau$, and augmentation ratio $p$. Since $T$ and $\alpha$ are specific to the condensed graph generation, CGC-X contains only two hyper-parameters, making it the method with the fewest hyper-parameters among existing graph condensation approaches~\cite{xu2024survey}.
We examine the impact of these hyper-parameters on our method's performance, with the results presented in Figure \ref{fig_hyper}. A higher $T$ generally improves performance, indicating that a sparser adjacency matrix enhances node representation, in line with findings from~\cite{jin2022graph}. The optimal value for $\alpha$ should be selected to balance node smoothness with feature reconstruction.
Additionally, $\tau$ controls the contribution of nodes in the aggregation, and more complex datasets like Flickr benefit from a smaller $\tau$, which emphasizes the reliability of node scores. The augmentation ratio $p$ determines the number of augmented nodes, with increased augmentations leading to better results. However, an excessive number of augmented nodes can degrade performance and slow down the condensation process.
It is worth noting that the rapid condensation process of our method could significantly simplify hyper-parameter tuning, emphasizing the practical utility and superior effectiveness of our proposed method.

\begin{figure}[t]
\setlength{\abovecaptionskip}{1pt}
\centering
\begin{minipage}[t]{0.95\linewidth}
\centering
\includegraphics[width=\linewidth]{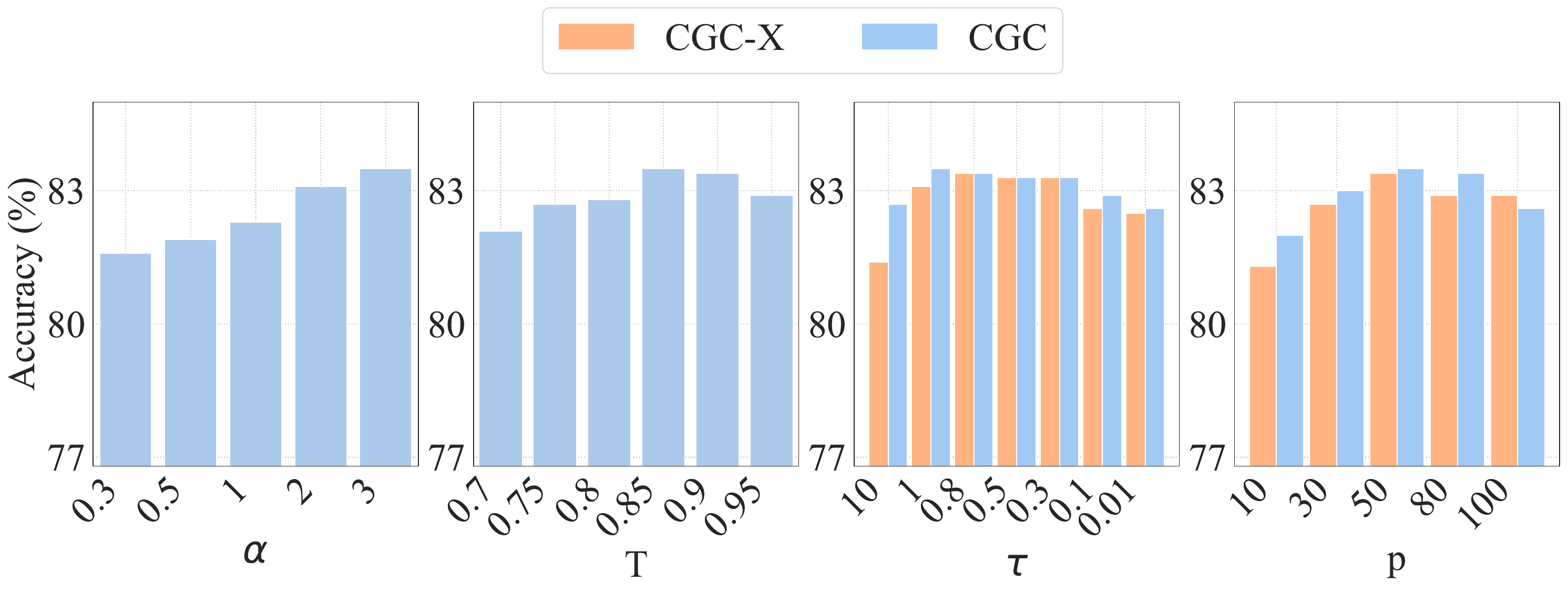}
\end{minipage}
\begin{minipage}[t]{0.95\linewidth}
\centering
\includegraphics[width=\linewidth]{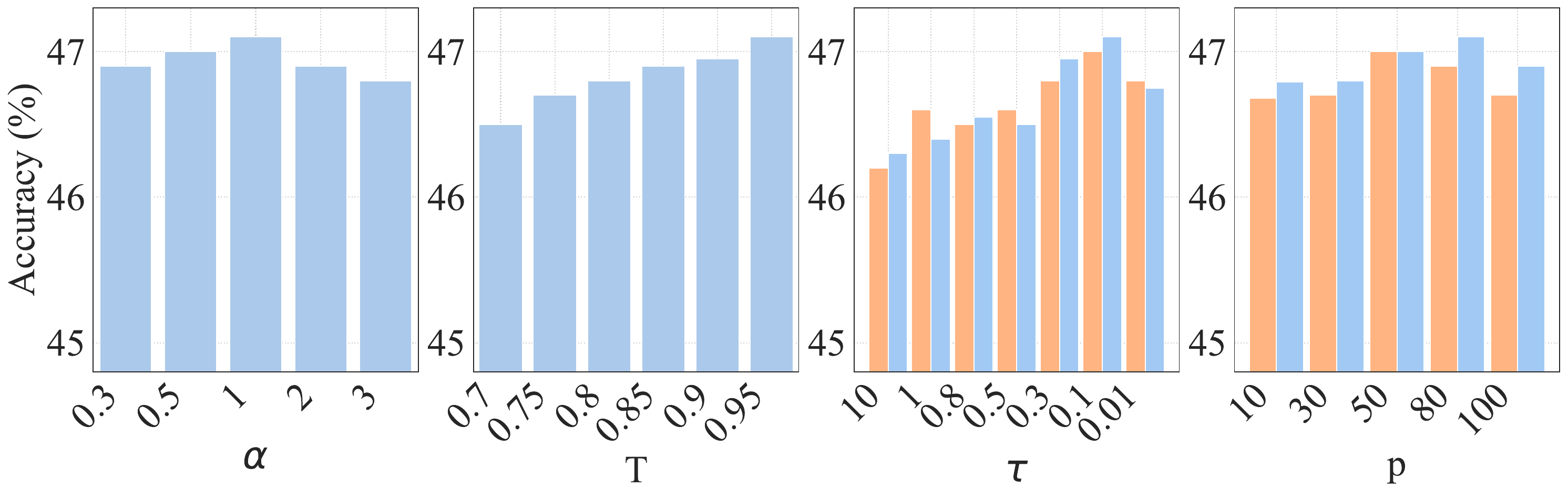}
\end{minipage}
\caption{Test accuracy across varying hyper-parameters: the first row shows results for Cora ($r=2.60\%$), and the second row for Flickr ($r=0.50\%$).}
\label{fig_hyper}
\end{figure}

%% file: 6con.tex
\section{Conclusion}
\label{sec_conclusion}
In this paper, we present CGC, a training-free GC framework designed for efficient condensed graph generalization. 
CGC transforms the class-level distribution matching paradigm identified in existing GC methods into a class partition problem, enabling the EM-based clustering solution for complex condensation optimization. 
CGC incorporates the pre-defined graph structure and closed-form feature solution, facilitating efficient condensed graph generation.
{{It is worth noting that the simplified paradigm in CGC-X also bridges GC and graph pooling methods, inspiring future work to integrate advanced pooling techniques for improved performances.}}

\begin{acks}
This work is supported by Australian Research Council under the streams of Future Fellowship (Grant No. FT210100624), Linkage Project (Grant No. LP230200892), Discovery Early Career Researcher Award (Grants No. DE230101033), and Discovery Project (Grants No. DP240101108 and No. DP240101814).
\end{acks}

%% file: 7app.tex
\newpage
\appendix
\begin{center}
\Large \bf {Appendix}
\end{center}
\etocdepthtag.toc{mtappendix}
\etocsettagdepth{mtchapter}{none}
\etocsettagdepth{mtappendix}{subsection}
\tableofcontents

\section{{{Notation Table}}}
{{To enhance readability, the main notations of our proposed method, along with their explanations, are summarized in Table \ref{nota}. The prime symbol denotes notations associated with the condensed graph.}}

\begin{table}[h]
\setlength{\abovecaptionskip}{1pt}
\renewcommand{\arraystretch}{1.2}
\caption{{{The notations and explanations.}}}
\label{nota}
\begin{tabular}{l|l}
\Xhline{1.pt}
Notation                 & Explanation                                                                                                                                   \\ \hline
${\bf Z}$ and ${\bf Z}'$ & Node embeddings.                                                                                                                              \\ \hline
${\bf H}$ and ${\bf H}'$ & Propagated features.                                                                                                                          \\ \hline
${\bf P}$ and ${\bf P}'$ & Linear aggregation matrices.                                                                                                                  \\ \hline
$\hat{\bf{P}}$           & Aggregation matrix with class partition constraint.                                                                                           \\ \hline
$C_i$ and $C'_i$         & Class $i$.                                                                                                                                    \\ \hline
$S^i_j$                  & \begin{tabular}[c]{@{}l@{}}Subclass $j$ for the class $i$, which is derived by\\ performing class partition on class $i$.\end{tabular}        \\ \hline
${\bf{R}}^i$             & \begin{tabular}[c]{@{}l@{}}The aggregation matrix for class $i$, \\ which is combined to form $\hat{\bf{P}}$.\end{tabular}           \\ \hline
${{\bf{R}}}^{i}_{j,k}$   & \begin{tabular}[c]{@{}l@{}}The aggregation weight for node $k$ in class $i$. It is \\ determined by the size of subclass $j$, if $k$ belongs to \\ subclass $j$.\end{tabular} \\ \hline
$\pi^i(k)$               & Subclass label for node $k$.                                                                                                                  \\ \Xhline{1.pt}
\end{tabular}
\end{table}

\section{Proof of Propositions}
\label{sec_apppro}
\setcounter{proposition}{0}

\subsection{Proof of Proposition 1}
\label{app_pro1}

\begin{proposition}
The performance matching objective is equivalent to the optimal parameter matching objective.
\end{proposition}

The proof of Proposition 1 builds on the proposition from the dataset distillation survey in CV~\cite{yu2023dataset}. To ensure self-containment, we detail and expand propositions within the context of graph theory.

\begin{proof}
Performance matching introduce KRR in optimization and objectives for ${\mathcal{T}}$ and ${\mathcal{S}}$ are formulated as:
\begin{equation}
\begin{aligned}
&\arg\min_{\mathbf{\Theta} }\left \| {\bf Z}{\mathbf{\Theta} } - {\bf Y} \right \|^{2}+\lambda \left \|{\mathbf{\Theta} } \right \|^{2},\\
&\arg\min_{\mathbf{\Theta} }\left \| {\bf Z}'{\mathbf{\Theta} } - {\bf Y}' \right \|^{2}+\lambda \left \|{\mathbf{\Theta} } \right \|^{2},
\end{aligned}
\end{equation}
where ${\bf Z}$ and ${\bf Z}'$ are node embeddings for the original graph and the condensed graph, respectively. ${\mathbf{\Theta} }$ is the learnable weight matrix and $\lambda$ is a small constant weight of the regularization term for numerical stability. 
Their closed-form optimal solutions are:
\begin{equation}
\begin{aligned}
&{\mathbf{\Theta} }^{*} ={\bf Z}^{\top}\left({\bf Z}{\bf Z}^{\top} + \lambda \mathbf{I}\right)^{-1} {\bf Y}, \\
&{\mathbf{\Theta} }'^{*} ={\bf Z}'^{\top}\left({\bf Z}'{\bf Z}'^{\top} + \lambda \mathbf{I}\right)^{-1} {\bf Y}'.
\end{aligned}
\end{equation}

We simplifies the regularization term in KRR model and assume $\lambda$=0. 
Then, the optimal parameter matching objective, expressed in the form of a least squares function, is formulated as:
\begin{equation}
\label{eq_PMPM}
\begin{aligned}
\mathcal{L}_{PM} &= \arg\min_{{\bf Z}'}\left \| {\mathbf{\Theta} }^{*}  - {\mathbf{\Theta} }'^{*}  \right \|^{2},\\
&= \arg\min_{{\bf Z}'}\left \| {\bf Z}^{\top}\left({\bf Z}{\bf Z}^{\top} \right)^{-1} {\bf Y} -  {\bf Z}'^{\top}\left({\bf Z}'{\bf Z}'^{\top}\right)^{-1} {\bf Y}'  \right \|^{2},\\
&= \arg\min_{{\bf Z}'}\left \| \left({\bf Z}^{\top}{\bf Z} \right)^{-1}{\bf Z}^{\top} {\bf Y} -  {\bf Z}'^{\top}\left({\bf Z}'{\bf Z}'^{\top}\right)^{-1} {\bf Y}'  \right \|^{2},\\
&= \arg\min_{{\bf Z}'}\left \| \left({\bf Z}^{\top}{\bf Z} \right)^{-1}{\bf Z}^{\top} {\bf Y} - \left({\bf Z}^{\top}{\bf Z} \right)^{-1}{\bf Z}^{\top}{\bf Z} {\bf Z}'^{\top}\left({\bf Z}'{\bf Z}'^{\top}\right)^{-1} {\bf Y}'  \right \|^{2},\\
&= \arg\min_{{\bf Z}'}\left \| \left(\left({\bf Z}^{\top}{\bf Z} \right)^{-1}{\bf Z}^{\top}\right) \left({\bf Y} - {\bf Z} {\bf Z}'^{\top}\left({\bf Z}'{\bf Z}'^{\top}\right)^{-1} {\bf Y}' \right) \right \|^{2},\\
&=  \arg\min_{{\bf Z}'}\left \| {\bf Y} - {\bf Z} {\bf Z}'^{\top}\left({\bf Z}'{\bf Z}'^{\top}\right)^{-1} {\bf Y}'  \right \|^{2}.
\end{aligned}
\end{equation}
The final step is justified by the independence of  $\left({\bf Z}^{\top}{\bf Z} \right)^{-1}{\bf Z}^{\top}$ from ${\bf Z}'$. Then, the r.h.s of Eq. (\ref{eq_PMPM}) is the performance matching objective. 

\end{proof}

\subsection{Proof of Proposition 2}
\label{app_pro2}

\begin{proposition}
The distribution matching objective represents a simplified formulation of the performance matching, omitting feature correlation considerations.
\end{proposition}

\begin{proof}
According to Proposition 1, the objective of performance matching can be formulated as:
\begin{equation}
\label{eq_PM}
\begin{aligned}
\mathcal{L}_{PM} 
&= \arg\min_{{\bf Z}'}\left \| \left({\bf Z}^{\top}{\bf Z} \right)^{-1}{\bf Z}^{\top} {\bf Y} -  \left({\bf Z}'^{\top}{\bf Z}'\right)^{-1}{\bf Z}'^{\top} {\bf Y}'  \right \|^{2},\\
&= \arg\min_{{\bf Z}'}  \sum_{i=1}^{c} \left \| \left({\bf Z}_{i}^{\top}{\bf Z}_{i} \right)^{-1}{\bf Z}_{i}^{\top} {\bf Y}_{i} -  \left({\bf Z}_{i}'^{\top}{\bf Z}_{i}'\right)^{-1}{\bf Z}_{i}'^{\top} {\bf Y}_{i}'  \right \|^{2},
\end{aligned}
\end{equation}
where subscript $i$ denotes the class, with ${\bf Z}_{i}$ and ${\bf Z}'_{i}$ representing original and condensed node embeddings in class $i$. ${\bf Y}_{i}$ and ${\bf Y}'_{i}$ representing one-hot labels for nodes in class $i$.
The distribution matching objective is calculated for class prototypes separately and formulated as:
\begin{equation}
\label{eq_DMapp}
\begin{aligned}
\mathcal{L}_{DM} &= \arg\min_{{\bf Z}'} 
\left \| {\bf P}'{\bf{Z}}'- {\bf P}{\bf{Z}} \right \|^{2},\\
&= \arg\min_{{\bf Z}'}  \sum_{i=1}^{c} \left \| \frac{1}{\left | C_i \right | }  {\bf Y}_{i}^{\top} {\bf Z}_{i} -  \frac{1}{\left | C'_i \right | }  {\bf Y}_{i}'^{\top} {\bf Z}_{i}' \right \|^{2},\\
&= \arg\min_{{\bf Z}'}  \sum_{i=1}^{c} \left \| \frac{1}{\left | C_i \right | }   {\bf Z}_{i}^{\top}{\bf Y}_{i} -  \frac{1}{\left | C'_i \right | }   {\bf Z}_{i}'^{\top} {\bf Y}_{i}' \right \|^{2},
\end{aligned}
\end{equation}
where ${\left | C'_i \right | }$ and ${\left | C_i \right | }$ represent the sizes of class $i$ in $\mathcal{S}$ and $\mathcal{T}$, respectively.
By comparing Eq. (\ref{eq_PM}) and Eq. (\ref{eq_DMapp}), it can be deduced that the performance matching objective can be transformed into distribution matching by excluding feature correlations and incorporating class size normalization.

\end{proof}

\subsection{Proof of Proposition 3}
\label{app_pro3}

\begin{proposition}
The distribution matching objective with a feature correlation constraint provides an upper bound for the parameter matching objective.
\end{proposition}

The proof of Proposition 3 is derived by extending the proposition in ~\cite{yu2023dataset}.

\begin{proof}
Without loss of generality, we take the representative gradient matching objective utilized in GCond for illustration, where SGC is utilized as the relay model. Therefore, the node embeddings are represented by:
\begin{equation}
\label{eq_SGCapp}
{\bf{Z}}' = {\bf{H}}'{\mathbf{\Theta} }, \ \ \ \ \ \ \ {\bf{Z}} = {\bf{H}}{\mathbf{\Theta} }, \\
\end{equation}
where ${\mathbf{\Theta} }$ is the learnable weight matrix to transform propagated features ${\bf{H}}'$ and ${\bf{H}}$. 
Due to that the objective of gradient matching is calculated for each class separately, we define the losses for ${\mathcal{T}}$ and ${\mathcal{S}}$ as:
\begin{equation}
\begin{aligned}
&\mathcal{L}_{i}^{\mathcal{T}}  = \frac{1}{2}\left \| {\bf H}_{i}{\mathbf{\Theta} } - {\bf Y}_{i} \right \|^{2},  \\ 
&\mathcal{L}_{i}^{\mathcal{S}} = \frac{1}{2}\left \| {\bf H}'_{i}{\mathbf{\Theta} } - {\bf Y}'_{i} \right \|^{2},
\end{aligned}
\end{equation}
where ${\bf H}_{i}$ and ${\bf H}'_{i}$ represent original and condensed propagated features in class $i$. ${\bf Y}_{i}$ and ${\bf Y}'_{i}$ represent one-hot labels for nodes in class $i$. 
For simplicity, we specify the distance function $\mathcal{D}$ in the gradient matching objective as the normalized least squares function.
Consequently, the objective function of gradient matching is formulated as
{\footnotesize
\begin{equation}
\label{eq_PMDM}
\begin{aligned}
\mathcal{L}_{GM}  &= \sum_{i=1}^{c} \left \| \frac{1}{\left | C_i \right | } \nabla_{{\mathbf{\Theta} }} \mathcal{L}_{i}^{\mathcal{T}} - \frac{1}{\left | C'_i \right | } \nabla_{{\mathbf{\Theta} }} \mathcal{L}_{i}^{\mathcal{S}} \right \|^{2},\\
&= \sum_{i=1}^{c} \left \| \frac{1}{\left | C_i \right | } ( {\bf H}_{i}^{\top} {\bf H}_{i}{\mathbf{\Theta} } -  {\bf H}_{i}^{\top} {\bf Y}_{i} )              - \frac{1}{\left | C'_i \right | } ( {\bf H}_{i}'^{\top} {\bf H}_{i}'{\mathbf{\Theta} } -  {\bf H}_{i}'^{\top} {\bf Y}'_{i} ) \right \|^{2},\\
&= \sum_{i=1}^{c} \left \| \frac{1}{\left | C_i \right | } {\bf H}_{i}^{\top} {\bf H}_{i}{\mathbf{\Theta} }   - \frac{1}{\left | C'_i \right | } {\bf H}_{i}'^{\top} {\bf H}_{i}'{\mathbf{\Theta} }  -  \frac{1}{\left | C_i \right | } {\bf H}_{i}^{\top} {\bf Y}_{i}      +  \frac{1}{\left | C'_i \right | } {\bf H}_{i}'^{\top} {\bf Y}'_{i} \right \|^{2},\\
&\le \sum_{i=1}^{c} \left\|  \frac{1}{\left | C_i \right | } {\bf H}_{i}^{\top} {\bf Y}_{i} -  \frac{1}{\left | C'_i \right | } {\bf H}_{i}'^{\top} {\bf Y}'_{i} \right \|^{2}+ \sum_{i=1}^{c}\left \| \frac{1}{\left | C_i \right | } {\bf H}_{i}^{\top} {\bf H}_{i} - \frac{1}{\left | C'_i \right | } {\bf H}_{i}'^{\top} {\bf H}_{i}' \right \|^{2}  \left \| {\mathbf{\Theta} } \right \|^{2},\\
&= \sum_{i=1}^{c} \left\|  \frac{1}{\left | C_i \right | }  {\bf Y}_{i}^{\top}{\bf H}_{i} - \frac{1}{\left | C'_i \right | }  {\bf Y'}_{i}^{\top}{\bf H}_{i}' \right \|^{2}+ \sum_{i=1}^{c}\left \| \frac{1}{\left | C_i \right | } {\bf H}_{i}^{\top} {\bf H}_{i} - \frac{1}{\left | C'_i \right | } {\bf H}_{i}'^{\top} {\bf H}_{i}' \right \|^{2}  \left \| {\mathbf{\Theta} } \right \|^{2}.
\\
\end{aligned}
\end{equation}}
The first term in r.h.s is the distribution matching objective and the second term quantifies differences in class-wise feature correlations.

\end{proof}

\subsection{Proof of Proposition 4}
\label{app_pro4}

\begin{proposition}
Assume an undirected condensed graph $\mathcal{S}=\{{\bf A'}, {\bf X'}\}$, the closed-form solution of Eq. (\ref{eq_feat}) take the form: ${\bf X}'=({\bf Q}^{\top}{\bf Q}+\alpha {\bf L}')^{-1}{\bf Q}^{\top}{\bf H}'$, where ${\bf Q} = \hat{\bf{A}}'^{K}$.
\end{proposition}

The proof of Proposition 4 is derived by extending the proposition in \cite{kumar2023featured}.

\begin{proof}
Given an undirected graph $\mathcal{S} = \{\textbf{X}',\textbf{A}'\}$, $\textbf{L}' \in \mathbb{R}^{N' \times N'}$ is defined as the Laplacian matrix of $\mathcal{S}$ by $\textbf{L}'= \textbf{D}' - \textbf{A}'$, where $\textbf{D}'$ is the degree matrix.
To solve the optimization problem in Eq. (\ref{eq_feat}), we first calculate $\nabla \mathcal{L}({\bf{X}}')$ and $\nabla^2 \mathcal{L}({\bf{X}}')$ as:
\begin{equation}
\begin{aligned}
\nabla \mathcal{L}({\bf{X}}')&=2{\bf Q}^{\top}({\bf Q}{\bf{X}}'-{\bf{H}}')+ \alpha ({\bf L}'+{\bf L}'^{\top}){\bf{X}}'\\
&=2({\bf Q}^{\top}{\bf Q}+\alpha{\bf L}'){\bf{X}}'-2{\bf Q}^{\top}{\bf{H}}'.
\end{aligned}
\end{equation}
\begin{equation}
\begin{aligned}
\nabla^2 \mathcal{L}({\bf{X}}')=2({\bf Q}^{\top}{\bf Q}+\alpha{\bf L}').
\end{aligned}
\end{equation}
According to the definition of $\textbf{A}'$ in Eq. (\ref{eq_adj}), $\textbf{L}'$ and ${\bf Q}$ are the positive semi-definite matrices.
Therefore, $\nabla^2 \mathcal{L}({\bf{X}}')$ is the positive semi-definite matrix and the optimization problem is a convex optimization problem. We can get the closed-form solution by calculate $\nabla \mathcal{L}({\bf{X}}')=0$ and ${\bf{X}}' = ({\bf Q}^{\top}{\bf Q}+\alpha{\bf L}')^{-1}{\bf Q}^{\top}{\bf{H}}'$. 

\end{proof}

%% file: 7app2.tex
\section{Algorithm}
\label{sec_appalg}

The detailed algorithm of CGC and CGC-X is shown in Algorithm \ref{al}. In detail, we first propagate node features according to the graph structure via non-parametric propagation methods. Consequently, a linear model is generated to assess embeddings and augmented features are sampled according to evaluation results. Then, embeddings in each class are partitioned by clustering method (i.e., $\text{Clustering}(\cdot)$ in line 8) to generate aggregated embeddings ${\bf{H}}'$. ${\bf{H}}'$ can serve as the condensed node feature for CGC-X.
If condensed graph structure is required, ${\bf A}'$ and ${\bf X}'$ can be generated by Eq. (\ref{eq_adj}) and Eq. (\ref{eq_feat}), respectively.

\begin{algorithm}[t]
\SetAlgoVlined
\DontPrintSemicolon 
\textbf{Input:} Original graph \(\mathcal{T}=\{{\bf A}, {\bf X}\}\), \({\bf Y}\) and pre-defined condensed graph label \({\bf Y}'\) \\
\textbf{Output:} Condensed graph \(\mathcal{S}\)\\
\BlankLine
\tcc{Feature Propagation}
Generate propagated features \({\bf{H}}^{(l)}\) according to Eq. (\ref{eq_FP}).\\
\tcc{Data Assessment}
Generate linear model and evaluate generated features.\\
\tcc{Data Augmentation}
Sample features according to class prediction errors.\\
\tcc{Class Partition}
Categories \({\bf{H}}_{cond}\) into \(c\) classes \(\{\hat{\bf{H}}_1, \cdots, \hat{\bf{H}}_c\}\) according to \({\bf{y}}_{cond}\).\\
\For{\(i=1,\ldots,c\)}  
{
Compute \({\bf{H}}'_i = \text{Clustering}(\hat{\bf{H}}_i)\). \\
}
Generate \({\bf{H}}'= [{\bf{H}}'_1; \ldots; {\bf{H}}'_c]\).\\
\tcc{Graph Generation}
\If{\rm{CGC-X}}
{
\textbf{Return:} \(\mathcal{S}=\{{\bf I}, {\bf H}'\}\)\\
}
\ElseIf{\rm{CGC}}
{
Generate \({\bf A}'\) according to Eq. (\ref{eq_adj}).\\
Generate \({\bf X}'\) by solving Eq. (\ref{eq_feat}).\\
\textbf{Return:} \(\mathcal{S}=\{{\bf A}', {\bf X}'\}\)\\
}
\caption{The framework of CGC and CGC-X.}
\label{al}
\end{algorithm}

\section{{{Complexity Analysis}}}
\label{sec_apptime}

The pipeline of CGC consists of five components: feature propagation, data assessment, data augmentation, class partition and graph generalization. 

We analyze the time complexity for each component separately. 
(1) Feature propagation contains the $K$ step graph convolution and time complexity is: $\mathcal{O} \left(KEd \right)$, where $d$ is the node dimension and $E$ is the number of edges.
(2) The solution of linear model involves the pseudo inverse of propagated features. In practice, this least squares problem in Eq. (\ref{eq_DA}) can be effectively solved by $\text{torch.linalg.lstsq}$ and the time complexity is $\mathcal{O} \left(Nd^2 \right)$, where $N$ denotes the number of nodes. In addition, the time complexity of node embedding prediction is $\mathcal{O} \left(dKNc \right)$, where  $c$  represents the number of classes.
(3) By leveraging propagated nodes with smaller depth as augmentations, the data augmentation module requires no additional computations. 
(4) Class partition executes the EM-based clustering on each class. Take K-means as an example, the time complexity for class $i$ is: $\mathcal{O} \left(\left | C'_i \right |\left | C_i \right |d t\right)$, where $t$ indicates the number of iteration for K-means. ${\left | C'_i \right | }$ and ${\left | C_i \right | }$ represent the size of class $i$ in $\mathcal{S}$ and $\mathcal{T}$ respectively.
(5) Graph generation contains the adjacency matrix generation and feature calculation. The former requires the pair-wise similarity calculation whose complexity is $\mathcal{O} \left(N'^2d \right)$. 
The closed-form solution of the node feature requires the inverse operation and the time complexity is $\mathcal{O} \left(N'^3+N'^2d+(K+1)N'E'+dE' \right)$, where $E'$ is the edge number of condensed graph.

In summary, the time complexity of CGC is dominated by class partition and node feature generation. However, due to the small size of the condensed graph and well-established clustering acceleration methods (e.g., FAISS~\cite{douze2024faiss}), CGC can be efficiently executed.

{{We also analyze the space complexity of each component as follows.
(1) Feature Propagation: The space complexity is $\mathcal{O}(E + (K+1)Nd)$, where $d$ is the node dimension, $N$ is the original graph size, $E$ represents the number of edges, and $K$ denotes the propagation depth.
(2) Data Assessment: The solution of the linear model entails a space complexity of $\mathcal{O}((K+1)Nd + Nc+dc)$, with $c$ being the number of classes.
(3) Data Augmentation: This process requires space complexity of $\mathcal{O}((K+1)Nd)$.
(4) Class Partitioning: This involves an additional partition matrix, resulting in a space complexity of $\mathcal{O}(N)$.
(5) Graph Generation: The generation of the condensed graph necessitates space complexity of $\mathcal{O}(N'^2 + N'd)$, where $N'$ is the size of the condensed graph.

Therefore, the total space complexity is $\mathcal{O}(E + (K+1)Nd + N'^2 + N'd)$.}}

\section{Experimental Setup Details}
\label{sec_appexpset}
\subsection{Dataset Statistics}
\label{sec_data}
We evaluate our proposed methods on four transductive datasets, i.e., Cora, Citeseer~\cite{DBLP:conf/iclr/KipfW17}, Ogbn-arxiv (Arxiv)~\cite{hu2020open} and Ogbn-products (Products)~\cite{hu2020open}, as well as two inductive datasets, i.e., Flickr and Reddit~\cite{DBLP:conf/iclr/ZengZSKP20}.
We adopt public splits throughout the experiments and dataset statistics are shown in Table \ref{tab_data}.

\begin{table*}[h]
\small
\centering
\setlength{\abovecaptionskip}{1pt}
\caption{The statistics of datasets used in experiments.}
\begin{tabular}{lrrrrrl}
\toprule
{Dataset}  & {\#Nodes} & {\#Edges} & {\#Classes} & {\#Features} &
{Training/Validation/Test} & {Task type}\\ \midrule
Cora  & 2,708 & 5,429 & 7 & 1,433 & 140/500/1,000& Transductive\\ 
Citeseer &  3,327 & 4,732 & 6 & 3,703 & 120/500/1,000 & Transductive\\ 
Ogbn-arxiv & 169,343 & 1,166,243 & 40 &  128 & 90,941/29,799/48,603 & Transductive \\
Ogbn-products & 2,449,029 & 61,859,140 & 47 & 100 & 196,615/39,323/2,213,091 & Transductive \\ 
Flickr & 89,250 & 899,756 & 7 & 500 & 44,625/22,312/22,313 & Inductive\\ 
Reddit &  232,965 & 23,213,838 & 41 & 602 & 153,932/23,699/55,334 & Inductive\\
\bottomrule
\end{tabular}
\label{tab_data}
\end{table*}

\subsection{Baselines}
\label{sec_baseline}

We compare our proposed methods against {{15}} baselines, including graph coarsening method and GC methods with diverse optimization strategies: 

\noindent(1) Variation Neighborhoods (VN)~\cite{loukas2019graph,huang2021scaling}. The conventional graph coarsening method that leverages the partition matrix to construct the super-nodes and super-edges. 

\noindent(2) A-ConvMatch (A-CM)~\cite{dickens2024graph}. A graph coarsening method hierarchically reduces the graph size while preserving the output of graph convolutions.

\noindent(3) GCond and GCond-X~\cite{jin2022graph}. The first GC method that utilizes the gradient matching to align the model parameters derived from both graphs.

\noindent(4) GCDM and GCDM-X~\cite{liu2022graph}. An efficient GC method that generates condensed graphs based on distribution matching by optimizing the maximum mean discrepancy between class prototypes.

\noindent(5) GC-SNTK and GC-SNTK-X~\cite{wang2023fast}. An efficient GC method leverages kernel ridge regression with a structure-based neural tangent kernel to simplify the bi-level optimization process.

\noindent(6) SFGC~\cite{zheng_structure_free_2023}. A graphless GC method that aligns long-term model learning behaviors through trajectory matching.

\noindent(7) SimGC~\cite{xiao2024simple}. An efficient GC method with the graph generation that introduces the pre-trained model in distribution matching.

\noindent(8) GEOM~\cite{zhang2024navigating}. A graphless GC method that explores the difficult samples via trajectory matching.

\noindent(9) GCSR~\cite{liu2024graph}. A trajectory matching-based GC method with the self-expressive condensed graph structure.

{{

\noindent(10) DosCond~\cite{liu2024graph}. A simplified gradient matching method with one-step model updating.

\noindent(11) SGDD~\cite{liu2024graph}. A generalized condensation method with enhanced graph structure modeling.

\noindent(12) MSGC~\cite{liu2024graph}. A gradient matching-based method with multiple sparse condensed graph structures.

Notice that the relay model for baselines are configured as described in their respective papers and released codes.
}}

\subsection{Implementations}
\label{sec_imple}
\textbf{Condensation Ratios.} 
We set $r$ following GCond. For transductive datasets,  $r$ is chosen as the \{25\%, 50\%, 100\%\}, \{25\%, 50\%, 100\%\}, {0.1\%, 0.5\%, 1\%\} and \{0.32\%, 0.63\%, 1.26\%\} of the labeled nodes for Cora, Citeseer, Arxiv and Products, respectively. For Inductive datasets, $r$ is set as \{0.1\%, 0.5\%, 0.1\%\} and \{0.05\%, 0.1\%, 0.2\%\} for Flickr and Reddit.

\noindent\textbf{Hyper-parameters.} 
The hyper-parameters are determined through the grid search on the validation set.
The feature proportion depth of the original graph and the condensed graph is set as $K=2$ for all datasets.  
The temperature $\tau$ is optimized from the set \{10, 1, 0.8, 0.5, 0.3, 0.1, 0.01\} and the augmentation ratio $p$ is searched  within the range [0, 100].
For the two hyper-parameters specific to CGC, the weight $\alpha$ is tuned from the set \{0.3, 0.5, 1, 2, 3\}, and the threshold $T$ is explored within the range [0.8, 1).
We use the ADAM optimization algorithm to train all the models. The learning rate for the condensation process is determined through a search over the set \{0.01, 0.001, 0.0001\}. The weight decay is 5e-4. Dropout is searched from [0, 1).

\noindent\textbf{Computing Infrastructure.} The codes are written in Python 3.9 and Pytorch 1.12.1. 
CGC and CGC-X are executed on CPUs and all GNN models are trained on GPUs. 
The experiments for Ogbn-products were conducted on a server equipped with Intel(R) Xeon(R) Gold 6326 CPUs at 2.90GHz and NVIDIA GeForce A40 GPUs with 48GB of memory. Other experiments were carried out on a server featuring Intel(R) Xeon(R) Gold 6128 CPUs at 3.40GHz and NVIDIA GeForce RTX 2080 Ti GPUs with 11GB of memory.

%% file: 7app3.tex
\begin{table}[th]
\renewcommand{\arraystretch}{1.2}
\setlength{\abovecaptionskip}{1pt}
\center
\caption{{{The generalizability of GC methods with graph generation on Flickr and Reddit datasets. SNTK is out-of-memory on Reddit. AVG indicates the average value. The best performances are highlighted.}}}
\label{tab_gene22}
\resizebox{\linewidth}{!}{
\begin{tabular}{c|l|rrrrrr|r}
    \Xhline{1.pt}
Dataset ($r$)                                                               & Method                          & \multicolumn{1}{l}{SGC}      & \multicolumn{1}{l}{GCN}      & \multicolumn{1}{l}{SAGE}     & \multicolumn{1}{l}{APPNP}    & \multicolumn{1}{l}{Cheby}    & \multicolumn{1}{l|}{GAT}     & \multicolumn{1}{l}{AVG}      \\ \hline
                                                                            & GCond                           & 46.1                         & 47.1                         & 46.2                         & 45.9                         & 42.8                         & 40.1                         & 44.7                         \\
                                                                            & GCDM                            & 44.3                         & 46.8                         & 45.8                         & 45.2                         & 41.8                         & 41.9                         & 44.3                         \\
                                                                            & DosCond & 46.1 & 46.2 & 46.0 & 45.8 & 42.9 & 40.6 & 44.6 \\
                                                                            & SGDD    & 46.3 & 46.4 & 46.5 & 46.1 & 43.7 & 41.2 & 45.0 \\
                                                                            & MSGC    & 46.5 & 46.4 & 45.8 & 46.0 & 42.0 & 40.9 & 44.6 \\
                                                                            & SNTK                            & 45.7                         & 46.8                         & 45.9                         & 45.3                         & 41.3                         & 41.4                         & 44.4                         \\
                                                                            & SimGC                           & 43.4                         & 45.6                         & 44.4                         & 44.8                         & 42.8                         & 41.2                         & 43.7                         \\
                                                                            & GCSR                            & 46.3                         & 46.6                         & \textbf{46.6}                & 46.3                         & 44.9                         & 45.6                         & 46.1                         \\
\multirow{-9}{*}{\begin{tabular}[c]{@{}l@{}}Flickr\\ (0.50\%)\end{tabular}} & CGC                             & \textbf{47.3}                & \textbf{47.1}                & \textbf{46.6}                & \textbf{46.9}                & \textbf{45.7}                & \textbf{46.1}                & \textbf{46.6}                \\ \hline
                                                                            & GCond                           & 89.6                         & 89.6                         & 89.1                         & 87.8                         & 75.5                         & 60.2                         & 82.0                         \\
                                                                            & GCDM                            & 88.0                         & 89.7                         & 89.3                         & 88.9                         & 74.9                         & 69.3                         & 83.4                         \\
                                                                            & DosCond & 89.2 & 90.5 & 89.8 & 88.1 & 76.1 & 62.2 & 82.7 \\
                                                                            & SGDD    & 90.2 & 90.7 & 90.1 & 88.7 & 77.6 & 63.2 & 83.4 \\
                                                                            & MSGC    & 90.0 & 89.2 & 89.5 & 88.3 & 76.3 & 61.2 & 82.4 \\
                                                                            & SimGC                           & 90.8                         & 90.6                         & 86.2                         & 88.6                         & 76.2                         & 65.1                         & 82.9                         \\
                                                                            & GCSR                            & 91.0                         & 91.2                         & \textbf{91.0}                & \textbf{88.9}                & 80.4                         & 86.4                         & 88.2                         \\
\multirow{-8}{*}{\begin{tabular}[c]{@{}l@{}}Reddit\\ (0.10\%)\end{tabular}} & CGC                             & \textbf{91.3}                & \textbf{91.4}                & 90.2                         & 88.7                         & \textbf{81.7}                & \textbf{89.1}                & \textbf{88.7}                \\     \Xhline{1.pt}
\end{tabular}
}
\end{table}

\section{Generalizability Comparison}

\subsection{Graph Generation Methods}
\label{sec_appgenedata}

Besides the results on the Arxiv dataset, we also evaluate our CGC and baselines on Flickr and Reddit datasets and results are shown in Table \ref{tab_gene22}. Similarly, CGC achieves a significant improvement over other compared baselines and demonstrates the great generalizability.

\subsection{Graphless Methods}
\label{sec_appgene}

We compare the GNN architecture generalizability of graphless GC methods, where $\{\bf I, \bf X'\}$ is used in GNN training.
Due to the absence of condensed graph structure, we evaluate GNNs including GCN, SGC, SAGE, APPNP and Cheby. 
According to the results in Table \ref{tab_geneX}, we could observe that our proposed CGC-X achieves the best generalizability across different datasets and architectures. Especially on the dataset Flickr, where the average improvements are 2.3\%. 
This verifies the high quality of the condensed feature and the effectiveness of the class-to-node distribution matching strategy.

\begin{table}[th]
\setlength{\abovecaptionskip}{1pt}

\center
\caption{The generalizability comparison of graphless GC methods. SNTK-X is out-of-memory on Reddit dataset. AVG indicates the average value. The best performances are highlighted.}
\label{tab_geneX}
\resizebox{\linewidth}{!}{
\begin{tabular}{l|l|rrrrr|r}
\Xhline{1.pt}
Dataset ($r$)                                                                   & Method  & \multicolumn{1}{l}{SGC} & \multicolumn{1}{l}{GCN} & \multicolumn{1}{l}{SAGE} & \multicolumn{1}{l}{APPNP} & \multicolumn{1}{l|}{Cheby} & \multicolumn{1}{l}{AVG} \\ \hline
\multirow{5}{*}{\begin{tabular}[c]{@{}l@{}}Arxiv\\ (0.25\%)\end{tabular}}  & GCond-X & 64.7                    & 64.2                    & 64.4                     & 61.5                      & 59.5                       & 62.9                    \\
                                                                           & GCDM-X  & 64.4                    & 61.2                    & 63.4                     & 60.5                      & 60.2                       & 61.9                    \\
                                                                           & SNTK-X  & 64.0                    & 65.5                    & 62.4                     & 61.8                      & 58.7                       & 62.5                    \\
                                                                           & SFGC    & 64.8                    & 66.1                    & 64.8                     & \textbf{63.9}             & 60.7                       & 64.1                    \\
                                                                           & GEOM    & 65.0                    & 67.5                    & 64.9                     & 62.5                      & 60.9                      & 64.2                    \\
                                                                           & CGC-X   & \textbf{66.3}           & \textbf{66.3}           & \textbf{65.9}            & 63.0                      & \textbf{61.3}              & \textbf{64.6}           \\ \hline
\multirow{5}{*}{\begin{tabular}[c]{@{}l@{}}Flickr\\ (0.50\%)\end{tabular}} & GCond-X & 44.4                    & 45.0                    & 44.7                     & 44.6                      & 42.3                       & 44.2                    \\
                                                                           & GCDM-X  & 45.8                    & 45.6                    & 42.6                     & 42.4                      & 42.4                       & 43.8                    \\
                                                                           & SNTK-X  & 44.0                    & 46.7                    & 41.9                     & 41.0                      & 41.4                       & 43.0                    \\
                                                                           & SFGC    & 47.0                    & \textbf{47.0}           & 42.5                     & 40.7                      & 45.4                       & 44.5                    \\
                                                                           & GEOM    & 46.5                    & 46.2                    & 42.7                     & 42.6                      & 46.1                      & 44.8                    \\
                                                                           & CGC-X   & \textbf{47.2}           & \textbf{47.0}           & \textbf{46.6}            & \textbf{46.8}             & \textbf{46.4}              & \textbf{46.8}           \\ \hline
\multirow{4}{*}{\begin{tabular}[c]{@{}l@{}}Reddit\\ (0.10\%)\end{tabular}} & GCond-X & 91.0                    & 89.3                    & 89.3                     & 78.7                      & 74.0                       & 84.5                    \\
                                                                           & GCDM-X  & 90.9                    & 87.2                    & 89.2                     & 79.1                      & 75.1                       & 84.3                    \\
                                                                           & SFGC    & 89.5                    & 90.0                    & 90.3                     & 88.3                      & 82.8                       & 88.2                    \\
                                                                           & GEOM    & 89.6                    & 90.4                    & 90.5                     & 89.4                      & 82.6                      & 88.5                    \\
                                                                           & CGC-X   & \textbf{91.5}           & \textbf{90.8}           & \textbf{90.9}            & \textbf{89.0}             & \textbf{82.6}              & \textbf{89.0}           \\ \Xhline{1.pt}
\end{tabular}
}
\end{table}

\section{{Ablation Study}}
\label{sec_appablation}

In addition to the CGC-X, the ablation study results for CGC are presented in Table \ref{tab_abl2}. We assess the impact of data augmentation (w/o AUG), data assessment (w/o CAL), and partition methods (w/o K-means and with RAN). Both data augmentation and data assessment significantly influence the performance of CGC, particularly on label-sparse datasets. By deploying K-means, CGC maintains similar performance, indicating that it is not sensitive to the choice of partition method.

\section{{Propagation Method}}
\label{sec_apppropmethod}

We now represent the effect of different propagation methods in our proposed method. Besides the propagation method in SGC, we further evaluate two propagation methods, including personalized PageRank (PPR)~\cite{page1999pagerank} and SAGE~\cite{hamilton2017inductive}. The test accuracies are shown in Table \ref{tab_propmethod}. 
{{The results show that diverse propagation methods adapt to datasets with varying characteristics while maintaining comparable overall performance.}}
For instance, SAGE achieves better performance on Citeseer. PPR outperforms other methods on Reddit dataset. These results verify the generalization of our proposed method and the proper propagation method should be selected for different datasets.

\section{{Propagation Depth}}
\label{sec_apppropdep}
{{
To evaluate the effect of the propagation depth of the original graph on the performance of GNN models, we change the propagation depth of feature propagation and compare the performance of CGC-X. As shown in Figure \ref{fig_propdepth}, there is an initial increase in performance with increasing depth, followed by a subsequent decrease. The best propagation depth is 2, which is consistent with the layer of the evaluated GCN. 
Moreover, we could observe that the performance without the propagation is also competitive, which is aligned with the results in recent study PMLP~\cite{yang2022graph}.
}}

\begin{table}[t]
    \renewcommand{\arraystretch}{1.2}
    \setlength{\abovecaptionskip}{1pt}
    \center
    \caption{{{The ablation study of CGC. $r$ is set as 2.60\%, 1.80\%, 0.25\%, 0.50\% and 0.10\% for five datasets, respectively.}}}
    \label{tab_abl2}
    \resizebox{\linewidth}{!}{
    \begin{tabular}{l|cc ccc}
    \Xhline{1.pt}
    Method   & Cora     & Citeseer & Arxiv & Flickr& Reddit\\ \hline
    CGC w/o AUG       & 80.6\scriptsize{±0.8} & 71.1\scriptsize{±0.2} & 65.2\scriptsize{±0.1}       & 46.7\scriptsize{±0.1}       & 90.9\scriptsize{±0.1}       \\
    CGC w/o CAL       & 81.7\scriptsize{±0.6} & 71.6\scriptsize{±0.4} & 65.4\scriptsize{±0.2}       & 46.3\scriptsize{±0.1}       & 91.0\scriptsize{±0.1}       \\
    CGC w K-means      & 82.2\scriptsize{±1.5} & 72.1\scriptsize{±0.3} & \textbf{66.5\scriptsize{±0.2}}       & 46.9\scriptsize{±0.2}       & 91.3\scriptsize{±0.2}       \\
    CGC w RAN   & 79.8\scriptsize{±0.9} & 70.3\scriptsize{±0.6} & 61.2\scriptsize{±0.7} & 45.3\scriptsize{±0.1} & 89.7\scriptsize{±0.0} \\ \hline
    CGC      & \textbf{82.3\scriptsize{±1.3}} & \textbf{72.4\scriptsize{±0.2}} & 66.4\scriptsize{±0.1}       & \textbf{47.1\scriptsize{±0.1}}       & \textbf{91.4\scriptsize{±0.1}}       \\ \Xhline{1.pt}
\end{tabular}}
\end{table}

\begin{table}[t]
\setlength{\abovecaptionskip}{1pt}
\center
\caption{The performance comparison of different propagation methods. $r$ is set as 2.60\%, 1.80\%, 0.25\%, 0.50\% and 0.10\% for five datasets, respectively. OOM indicates the out-of-memory. The best performances are highlighted.}
\label{tab_propmethod}
\resizebox{0.99\linewidth}{!}
{
\begin{tabular}{l|lll|lll}
\Xhline{1.pt}
\multirow{2}{*}{Dataset} & \multicolumn{3}{c|}{CGC-X}                                & \multicolumn{3}{c}{CGC}                          \\ \cline{2-7} 
                         & SGC               & PPR               & SAGE              & SGC               & PPR      & SAGE              \\ \hline
Cora                     & \textbf{83.4+0.4} & 81.0+0.5          & 81.1±0.5          & \textbf{82.3+1.3} & 80.9+0.7 & 80.0±1.8          \\
Citeseer                 & 72.6+0.2          & 70.5+2.0          & \textbf{73.0±0.2} & 72.4+0.2          & 70.5+2.0 & \textbf{72.9±0.3} \\
Arxiv                    & \textbf{66.3+0.3} & \textbf{66.3±0.2}          & 65.4+0.6          & \textbf{66.4+0.1} & 65.7+0.1 & 65.9+0.3          \\
Flickr                   & \textbf{47.0+0.1} & \textbf{47.0+0.1}          & 46.7+0.2          & \textbf{47.1+0.1} & 47.0+0.1 & 46.5+0.0          \\
Reddit                   & 90.8+0.0          & \textbf{90.9+0.1} & 88.7±0.7          & \textbf{91.4+0.1} & 91.0+0.1 & 89.5±0.3          \\ \Xhline{1.pt}
\end{tabular}}
\end{table}

\begin{figure}[ht]
\setlength{\abovecaptionskip}{0.01cm}
\centering
\includegraphics[width=0.8\linewidth]{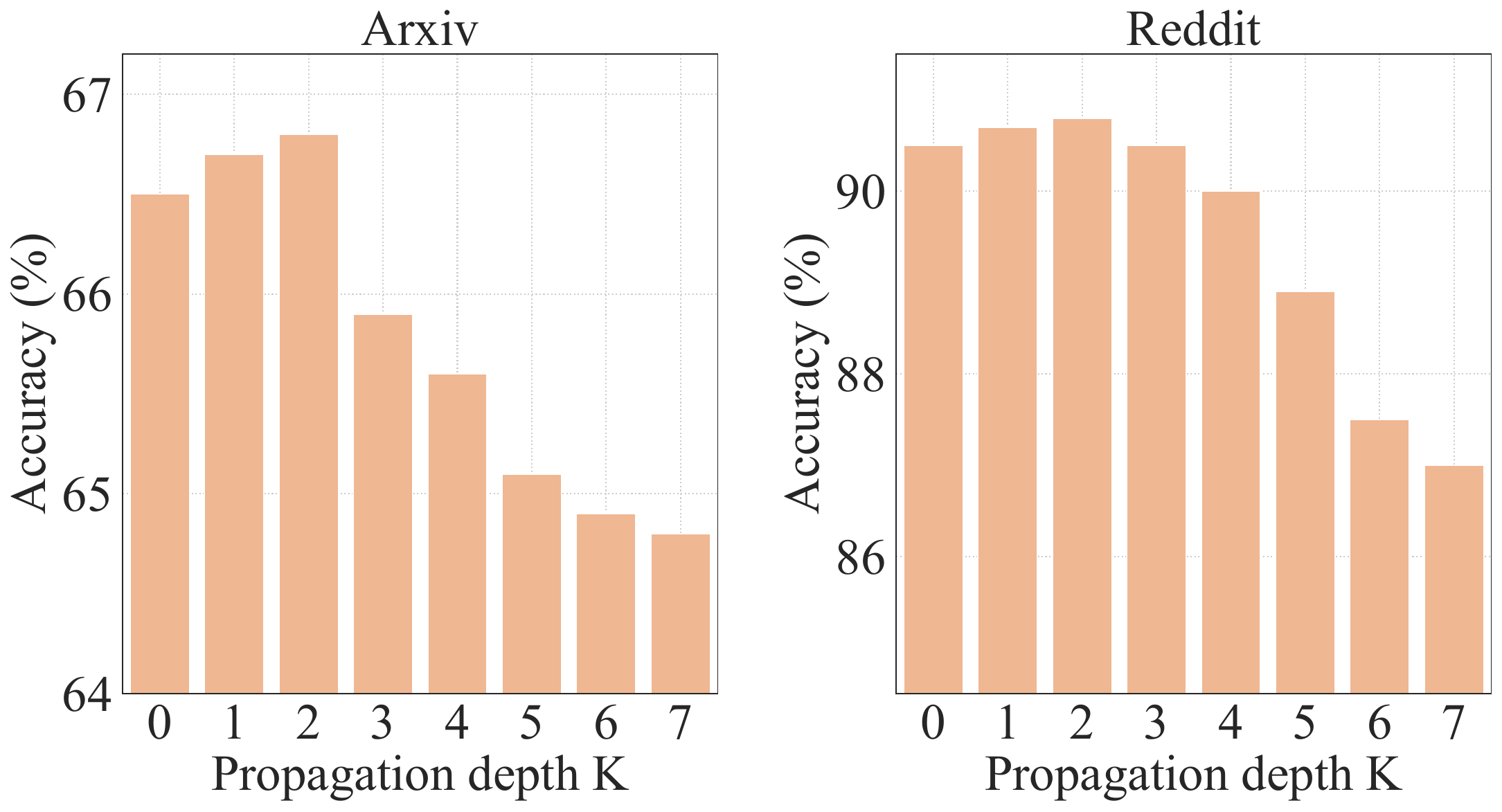}
\caption{{{The performance on different feature propagation depths. GCN is trained on the condensed graph generated by CGC-X. $r$ is set as 0.25\% and 0.10\% for Arxiv and Reddit, respectively.}}}
\label{fig_propdepth}
\end{figure}

\input{5rel}

%% file: 5rel.tex
\section{Related Work}

\noindent\textbf{Graph Condensation.}
GC~\cite{xu2024survey,gao2024graph,hashemi2024comprehensive, wang2024self} has recently garnered significant attention as a data-centric approach.
Beyond the diverse optimization strategies discussed in Section \ref{sec_rethink}, other research efforts have focused on enhancing these optimization strategies~\cite{xiao2024disentangled,zhang2024two,zhang2024navigating,li2023attend,liu2024tinygraph,rasti2024gstam} or developing superior condensed graph structures~\cite{gao2023multiple, liu2024graph,gao_graph_2023} to improve condensed graph quality. 
Additionally, there is an increasing emphasis on enhancing other aspects of models trained on condensed graphs, such as generalizability~\cite{yang_does_2023,gao2024graphopen,gao2024contrastive}, fairness~\cite{feng_fair_2023,mao2023gcare}, robustness~\cite{gao2024robgc}, security~\cite{wu2024backdoor} and explainability~\cite{fang2024exgc}.

\noindent\textbf{Graph Sparsification \& Coarsening.}
Early explorations into graph size reduction primarily focus on graph sparsification~\cite{chen2021unified,hui2023rethinking} and coarsening techniques~\cite{cai2021graph}. These methods aim to eliminate redundant edges and merge similar nodes while preserving essential graph characteristics, such as the largest principal eigenvalues~\cite{loukas2018spectrally} and Laplacian pseudoinverse~\cite{bravo2019unifying}. Recently, FGC~\cite{kumar2023featured} incorporates node features into the coarsening process to improve node merging. However, these methods overlook node label information, resulting in sub-optimal generalization across different GNN architectures and downstream tasks. Additionally, VNG~\cite{si2022serving} generates compressed graphs by directly minimizing propagation errors. However, the compressed graph's utility is restricted to serving time and cannot be used for GNN training, limiting its applicability.